%% file: main.tex
\documentclass[sigconf]{aamas}

\input{AAMAS25/macro}
\usepackage{newfloat}
\newtheorem{claim}{Claim}
\newtheorem{lemma}{Lemma}
\newtheorem{theorem}{Theorem}
\newtheorem{definition}{Definition}
\newcommand{\Amax}{\mathbf{A}^{\star}}
\renewcommand{\epsilon}{\varepsilon}
\usepackage{balance} 

\makeatletter
\gdef\@copyrightpermission{
  \begin{minipage}{0.2\columnwidth}
   \href{https://creativecommons.org/licenses/by/4.0/}{\includegraphics[width=0.90\textwidth]{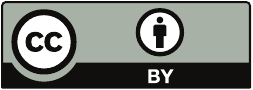}}
  \end{minipage}\hfill
  \begin{minipage}{0.8\columnwidth}
   \href{https://creativecommons.org/licenses/by/4.0/}{This work is licensed under a Creative Commons Attribution International 4.0 License.}
  \end{minipage}
  \vspace{5pt}
}
\makeatother

\setcopyright{ifaamas}
\acmConference[AAMAS '25]{Proc.\@ of the 24th International Conference
on Autonomous Agents and Multiagent Systems (AAMAS 2025)}{May 19 -- 23, 2025}
{Detroit, Michigan, USA}{Y.~Vorobeychik, S.~Das, A.~Nowé  (eds.)}
\copyrightyear{2025}
\acmYear{2025}
\acmDOI{}
\acmPrice{}
\acmISBN{}

\acmSubmissionID{201}

\title[Multi-agent Multi-armed Bandits with Minimum Reward Guarantee Fairness]{Multi-agent Multi-armed Bandits with Minimum Reward Guarantee Fairness}

\author{Piyushi Manupriya}
\affiliation{
  \institution{IIT Hyderabad}
  \city{Hyderabad}
  \country{India}}
\email{cs18m20p100002@iith.ac.in}

\author{Himanshu}
\affiliation{
\institution{IIT Hyderabad}
\city{Hyderabad}
\country{India}
}
\email{ai22mtech12008@iith.ac.in}

\author{SakethaNath Jagarlapudi}
\affiliation{
  \institution{IIT Hyderabad}
  \city{Hyderabad}
  \country{India}}
\email{saketha@cse.iith.ac.in}

\author{Ganesh Ghalme}
\affiliation{
  \institution{IIT Hyderabad}
  \city{Hyderabad}
  \country{India}}
\email{ganeshghalme@ai.iith.ac.in}

\input{AAMAS25/abstract}

\ccsdesc[500]{Theory of computation~Online learning algorithms}

\keywords{Multi-armed Bandits, Fairness, Regret Analysis, Multi-agent systems}

\newcommand{\BibTeX}{\rm B\kern-.05em{\sc i\kern-.025em b}\kern-.08em\TeX}

\begin{document}


\pagestyle{fancy}
\fancyhead{}


\maketitle 


\input{AAMAS25/intro}

\input{AAMAS25/prelims}

\input{AAMAS25/results}

\input{AAMAS25/simulation}

\input{AAMAS25/discussion}



\bibliographystyle{ACM-Reference-Format} 
\bibliography{main}
\newpage 
\appendix
\input{AAMAS25/Appendix}

\end{document}

%% file: AAMAS25/macro.tex
\usepackage{amsmath}
\usepackage{mathtools}
\usepackage{enumitem}
\usepackage{thm-restate}
\usepackage{amsfonts}
\usepackage{mathtools}
\usepackage{tcolorbox}
\usepackage{amsthm}
\usepackage{xcolor}
\usepackage{graphicx} 
\usepackage{dsfont}
\usepackage{algorithm}
\usepackage{booktabs}
\usepackage{bbm}
\usepackage{algorithmic}
\usepackage{color}
\usepackage{subcaption}

\newcommand{\ouralgo}{\textsc{RewardFairUCB}}

\DeclareMathOperator*{\argmax}{arg\,max}
\DeclareMathOperator*{\argmin}{arg\,min}

\newcount\Comments  
\definecolor{darkgreen}{rgb}{0,0.6,0}
\newcommand{\kibitz}[2]{\ifnum\Comments=1{\textcolor{#1}{{#2}}}\fi}
\newcommand{\gan}[1]{\kibitz{blue}{[GG:#1]}}
\newcommand{\himanshu}[1]{\kibitz{red}{[HIMANSHU:#1]}}
\newcommand{\piyushi}[1]{\kibitz{orange}{[Piyushi:#1]}}
	    \newcommand{\red}[1]{{\leavevmode\color{red}{#1}}}

\newcommand\tocite{\red{[CITE]}}

\newcommand\toref{\red{[REF]}}

\newtheorem{proposition}{Proposition}

\newtheorem{manualtheoreminner}{Theorem}
\newenvironment{manualtheorem}[1]{%
  \renewcommand\themanualtheoreminner{#1}%
  \manualtheoreminner
}{\endmanualtheoreminner}

%% file: AAMAS25/abstract.tex
\begin{abstract}
We investigate the problem of maximizing social welfare while ensuring fairness in a multi-agent multi-armed bandit (MA-MAB) setting. In this problem, a centralized decision-maker takes actions over time, generating random rewards for various agents. Our goal is to maximize the sum of expected cumulative rewards,  a.k.a. social welfare, while ensuring that each agent receives an expected reward that is at least a constant fraction of the maximum possible expected reward.

Our proposed algorithm, {\sc RewardFairUCB}, leverages the Upper Confidence Bound (UCB) technique to achieve sublinear regret bounds for both fairness and social welfare. The fairness regret measures the positive difference between the minimum reward guarantee and the expected reward of a given policy, whereas the social welfare regret measures the difference between the social welfare of the optimal fair policy and that of the given policy.

We show that {\sc RewardFairUCB} algorithm achieves  instance-independent social welfare regret guarantees of $\tilde{O}(T^{1/2})$ and a fairness regret upper bound of $\tilde{O}(T^{3/4})$. We also give the  lower bound of  $\Omega(\sqrt{T})$ for both social welfare and fairness regret. We evaluate { \sc RewardFairUCB}'s performance against various baseline and heuristic algorithms using simulated data and real world data, highlighting trade-offs between fairness and social welfare regrets. 

\end{abstract}

%% file: AAMAS25/intro.tex
\section{Introduction}

\gan{2 Pages}

\nocite*

In a classical stochastic Multi-armed Bandit (MAB) problem, a central decision maker takes an action or, equivalently,  selects an arm from a fixed set of arms at each of \( T \) time steps. Each arm pull yields a random reward from an unknown distribution. The objective is to develop a strategy for selecting arms that minimizes the regret; difference between the cumulative rewards of the best possible arm-pulling strategy in hindsight and the cumulative expected rewards achieved by the algorithm's policy.

\gan{PARA 2: MAMAB motivating examples}

\gan{Lets avoid statements with philosophical innuendos}

 We study an interesting variant of stochastic MAB problem, first proposed by \citet{Hossain2020FairAF} and known as  Multi-agent Multi Armed Bandits (MA-MAB). In the MA-MAB setting  an arm pull generates a vector-valued reward whose each entry is independently sampled from a fixed but unknown distribution denoting reward obtained by corresponding agent. When there is a single agent, this setting reduces to a classical stochastic MAB setting. 

The  MA-MAB setting captures several interesting real-world applications. Consider, for instance,  the problem of distributing a fixed monthly/yearly budget, say one unit, among $k$ different projects. There are \( n \) beneficiaries (or agents) who each experience varying levels of benefit from the different projects.    
Each agent $i \in [n]$ receives a reward sampled independently from distribution $\mathcal{D}(\mu_{i,j})$ when the algorithm pulls an arm $j$ (or equivalently, selects project $j$) where $\mu_{i,j}$ denotes the mean reward for agent $i$ from arm $j$. The randomness in the reward received by agents may arise from uncertainty in the assessment of the value of the project by individual agents and randomness in the aggregation/reporting step. Given  a distribution  $\pi \in \Delta_{m}$ over the  set $[m]$ of arms,  the total expected reward to agent $i$ is given by $\sum_{j \in [m]} \mu_{i,j} \pi_j$.

Consider another example where a networked TV channel must decide which movie/program to telecast on  a given time slot. The different movie/program genres are the arms, whereas the  population group (based on age group, demographics, etc.) are the agents. The reward represents the preferences of the corresponding agent, a.k.a. age group.  The decision-maker's problem here is to telecast the most liked program that, at the same time, caters to the preferences of a diverse population.  We will return to this example in Section \ref{sec: simulation}.


It is easy to see that when the goal is to maximize social welfare  \footnote{ Defined as the sum of cumulative expected reward.},  the resulting arm pull  strategy/allocation strategy  might become skewed. For example, consider a MA-MAB instance with \( n \) agents and \( m=2 \)  arms with  reward distributions such that \(\mu_{1,1} =1\) and \(\mu_{1,2} =0\), and \(\mu_{i,1} =0\) and \(\mu_{i,2} = 1/n\) for \(i >  1\). In this case, a social welfare maximizing  policy  would allocate the entire budget to the first project. Similarly,  select the movie genre most preferred by the entire population in the second example. However, this policy benefits only the first agent, leaving the vast majority of \(n-1\) agents without any reward.  Such winner-takes-all allocations can be considered unfair in many applications and can lead to undesirable long-term dynamics leading to mistrust towards the algorithm \cite{Hossain2020FairAF}.     
The  MA-MAB framework with fairness constraints facilitates the simultaneous optimization of both individual rewards and overall societal welfare.




\gan{para: introduce the fairness notion with intuitive explanation} 

 Consider a thought experiment where a single agent dictates the arm-pulling policy. \footnote{Alternatively, consider a scenario where there is only one agent, i.e., \( n = 1 \).}  This dictatorial agent would choose to pull her most rewarding arm at every time step. However, this policy completely disregards the preferences of other agents and thus fails to ensure any minimum reward for them. 
In this paper, we address the problem of ensuring a minimum reward guarantee for each agent as an explicit constraint. Specifically, each agent \( i \) is guaranteed at least a certain fraction \( C_i \in [0,1] \) of the maximum possible reward they could receive. 



\input{AAMAS25/Rel-work}


\gan{Para 6: Main results of the paper}
\subsection{Main Results and Organization of the Paper} We propose a novel formulation for the multi-agent multi-armed bandits (MA-MAB) problem to maximize social welfare obtained from the rewards while also guaranteeing each agent a specified fraction of their maximum possible reward. In Section~\ref{sec:settings_prelim}, we formally define the problem and provide sufficient conditions under which a fair MA-MAB instance is guaranteed to have a feasible solution. Then, in Section~\ref{sec:warmup}, we consider an MA-MAB instance with 2 arms and $n > 1$ agents and show that a simple {\sc Explore-First} algorithm achieves a simultaneous regret bound of $\tilde{O}(T^{2/3})$ for both fairness and social welfare. 

In Section~\ref{sec:UCB}, we propose the main algorithm of this paper, \ouralgo\ ,  and show that it achieves the regret guarantee of $\tilde{O}(\sqrt{T})$ for the social welfare regret and $\tilde{O}(T^{3/4})$ for the fairness regret. In Section~\ref{subsec:lowerBound}, we prove lower bounds of $\Omega (\sqrt{T})$ for both social welfare regret and fairness regret. These lower bounds hold independently for the regret notions. We then provide a dual formulation based heuristic algorithm in Section~\ref{sec:dual} that achieves a better regret performance on the simulated data and real-world datasets (Section~\ref{sec: simulation}).   The main results of the paper are summarized in  Table~\ref{tab:your_label}. 


\begin{table*}[ht!]
  \centering
  \begin{tabular}{@{}lcc  p{72mm}@{}}
    \toprule
      & Social Welfare Regret  & Fairness  Regret & \multicolumn{1}{c}{Remark}\\ \toprule
     Lower Bound (Sec~\ref{subsec:lowerBound})  & $\Omega(\sqrt{T})$  & $\Omega(\sqrt{T})$ &  The lower bounds hold individually for social welfare regret and fairness regret.     \\ \midrule
    \textsc{Explore-First}   (Sec~\ref{sec:warmup}) & $\tilde{O}(T^{2/3})$     & $\tilde{O}(T^{2/3})$ & For two arms, $C_i=c\leq  0.5\ \forall i\in [n]$.\\ \midrule
    \ouralgo \ (Sec~\ref{sec:UCB})  & $\tilde{O}(\sqrt{T})$  &   $\tilde{O}(T^{3/4})$  & For any finite number of arms. Optimal (up to logarithmic factor) social welfare regret.  \\ \bottomrule
  \end{tabular}
  \caption{Key findings of the paper. \piyushi{May be we should remove time horizon aware and in discussion write about doubling trick.}}
  \label{tab:your_label}
\end{table*}

%% file: AAMAS25/Rel-work.tex
\subsection{Related Work}
\label{sec:relWork}
\gan{@Piyushi: can you give a try at this?}
The stochastic multi-armed bandits (MAB) problem has been extensively studied with the goal of designing algorithms that optimally trade-off exploration and exploitation for maximizing the expected cumulative reward \cite{Lattimore2020BanditA, slivkins2024introductionmultiarmedbandits, bubeck2012regretanalysisstochasticnonstochastic}. The multi-agent multi-armed bandits (MA-MAB) variant \cite{Liu2009DistributedLI} involves multiple agents simultaneously solving a given instance of the MAB problem. Such a setting often demands providing reward fairness guarantees to each agent, besides maximizing the sum of expected cumulative rewards obtained by the agents. Several works have emerged focusing on fairness for the MAB problem \cite{JMLR:v22:20-704,wang2021fairness,sinha2023textttbanditq,baudry2024, porat21, liu2017calibrated, patil2022mitigating,krishna25pmean}. However, these approaches do not generalize to provide reward guarantees for different agents involved in the MA-MAB setup. Moreover, these formulations either focus on guaranteeing a certain fraction of arm pulls \cite{JMLR:v22:20-704,sinha2023textttbanditq, porat21},  constrain the deviation of the policy to a specific closed-form optimal policy \cite{wang2021fairness,baudry2024} or focus on meritocratic criteria in online resource allocation setting \cite{patil2022mitigating, liu2017calibrated}. 


The closest work to ours is by     \citet{Hossain2020FairAF} who proposed learning a policy over the $m$-arms that maximizes the Nash Social Welfare (NSW) involving the $n$ agents. \citet{Jones_Nguyen_Nguyen_2023} proposed a more efficient algorithm for the NSW-based MA-MAB problem and recently \citet{Zhang2024NoRegretLF} tried improving the corresponding regret bounds. While NSW objective is known to satisfy desirable fairness and welfare properties (see \cite{NSW, caragiannis2019unreasonable} for details), the fairness guarantees in NSW are implicit and cannot be specified externally. This may not always be desired. Consider a case with 2 agents with $\mu_{1, 1}=\mu_{2,2}=1$ and $\mu_{1, 2}=\mu_{2, 1}=0$ where the agents demand at least one-third and two-thirds of their maximum possible reward. Maximizing the NSW results in the policy $(1/2, 1/2)$ which does not satisfy the specified reward requirement for the   second agent. In contrast, our proposed MA-MAB formulation finds a policy that respects the reward allocations demanded by each agent, whenever it is possible to do so.

%% file: AAMAS25/prelims.tex
\section{Setting and Preliminaries}\label{sec:settings_prelim}

\gan{Para: Notation and basic setting }
We write $[n]$ to denote the set $\{1, 2, \cdots , n \}$. Further, we will assume that the set of agents ($[n]$) and the set of arms  ($[m]$) are both finite sets. 
Let $\mathcal{D}(\mu_{i,j})$ denote a  probability distribution with finite  mean $\mu_{i,j}$. Further, let the random variable  $X_{i,j}$ denote a random reward obtained by agent $i$ from arm $j$, that is  $X_{i,j} \sim \mathcal{D}(\mu_{i,j})$. Finally, we  $X_i$  to denote a vector of size $m$ with $i^{\rm{th}}$ entry  $ X_{i,j}$.

               A fair Multi-Agent Multi-Armed Bandit (MA-MAB)  instance $\mathcal{I}$ is denoted by a tuple $ \langle A, C, T \rangle $ where,  
\begin{itemize}
    \item $A$ denotes an $n \times m$ non-negative matrix with each entry $A_{i,j} := \mu_{i,j}$. Note that $A$ is fixed but unknown to the algorithm. We will assume, without loss of generality \footnote{ It is easy to see that if all the  entries  are divided by the largest row entry, the optimal strategy does not change.}  that $A_{i,j} \in [0,1]$ for all $i \in [n]$ and $j \in [m]$.  Further, define   $\Amax$ to be a matrix in $[0,1]^n$ whose  $i^{\rm{th}}$ entry represents the maximum possible expected reward  to agent $i$. $\Amax_i = \max_{j \in [m]} A_{i,j}$.  
    \item $C $ denotes $n \times n$ non-negative diagonal matrix. The entry $C_i : = C_{i,i} \in [0,1]$ specifies the fraction of maximum possible rewards to be guaranteed to agent $i$. Note that $C$ is a pre-defined constant and does not change with time. 
    \item $T$ denotes the stopping time of the algorithm. We assume that $T$ is known a priori. However,  all our results can be extended to unknown time horizon setting using a doubling trick (see \cite{besson2018doubling}) with an additional constant multiple factor increase in the regret.  
\end{itemize}

We now formally define the notion of minimum-reward guarantee fairness regret for a given MA-MAB instance $\mathcal{I}$. We begin by first defining the notion of minimum-reward guarantee.   

\begin{definition}[Minimum Reward Fairness Guarantee]
Let $\mathcal{I} = \langle A, C, T \rangle$ be  a MA-MAB instance and let  $\Amax$ be the vector of maximum values from the corresponding row of $A$. We say  that a policy $\pi$ satisfies minimum reward fairness guarantee for    $\mathcal{I}$,  if 
\begin{equation}
A \pi \geq C   \Amax .
\end{equation}  
\end{definition}

Throughout the paper, we  will call a policy $\pi$ fair if it satisfies minimum fairness guarantees for any fair MA-MAB instance $\mathcal{I}$.

We observe that   there may not always exist a fair policy. Consider, an example with 2 agents and 2 arms with $A_{1,1} = A_{2,2} = 1$ and $A_{1,2} = A_{2,2} = 0$ and $C_{1} = C_{2}  = c\in [0, 1] $. It is straightforward to see that  no policy $\pi$ satisfies the minimum reward fairness guarantee for   instances with $c>0.5$.  However, a fair policy always exists for $c \leq 0.5$. In particular, $\pi = [0.5, 0.5]$ is one such  fair policy. 

In our first result of the paper, we provide sufficient conditions that guarantee the existence of fair policy for a given instance $\mathcal{I}$

\begin{restatable}{theorem}{characterization}
\label{thm:characterization}
A fair MA-MAB instance $\mathcal{I}$ admits a fair policy if at least one of the below conditions is satisfied 
\begin{enumerate}
    \item $\sum_{i \in [n]} C_i \leq 1$, 
    \item $ C_{\max} := \max_{i\in [n]} C_i \leq \frac{1}{\min(n,m)}$. 
\end{enumerate}
\end{restatable}

For the first condition,  observe that the policy  $   \pi_j = \frac{\sum_{i=1}^n \mathds{1}(j = j_i) C_i}{\sum_{i=1}^n C_i}$ where $j_i$ being the arm  with largest reward to agent $i$ is a fair policy \footnote{If $\sum_{i} C_i = 0$, then every policy $\pi \in \Delta_m$ is feasible. }. Under the second condition, uniform arm pull policy i.e., $ \pi_j = [1/m, \cdots, 1/m]$ is feasible.    A formal proof of Theorem~\ref{thm:characterization} is given in the Appendix. 

\himanshu{ Sir, should we talk about the geometric interpretation of the feasibility region of $\pi$ we discussed when we were thinking from policy perspective instead of social welfare perspective}



    Under a learning setting where  the algorithm is not privy to $A$,  the algorithm must learn the policy ($\pi^t$) from the history of past pulls  and observed rewards, denoted by $\mathcal{H}^t$. More specifically,  for a given time instance $t$,  an arm pull strategy $\pi^t$ is a mapping,  $\pi^t: \mathcal{H}^t \rightarrow \Delta_{m}$.  The minimum-reward fairness regret of the policy $\pi:= (\pi^t)_{t\geq 1}$ is defined  the cumulative \emph{positive} difference between promised expected reward and the expected reward under the policy $\pi^t$.   

\begin{definition}[Minimum-reward Fairness Regret]
Given a MA-MAB instance $\mathcal{I}$ and a policy $\pi$, the minimum-reward fairness regret of $\pi$ on instance $\mathcal{I}$ over $T$ time instances is given as  
\begin{equation}
\mathcal{R}_{\textsc{fr}}^\pi(T) = \sum_{t=1}^T \sum_{i=1}^n | 
\underbrace{C_i \Amax_i}_{\text{I}} - \underbrace{\mathbb{E}_{\pi^t} [X_i^t]}_{\text{II}} |_+ \ ,
\end{equation}
where \( |\cdot|_+ \equiv \max\{\cdot, 0\} \).
\end{definition}

The term labelled \textbf{I} indicates the minimum rewards as a fraction of the maximum possible expected reward that agent \(i\) is guaranteed, while the term labelled \textbf{II} represents the expected reward that agent \(i\) receives under the policy \(\pi^t\) at time \(t\). The use of the expression \( |\cdot|_+ \) allows us to capture the scenario where, if the reward received by the agent exceeds the minimum required to satisfy the fairness constraints, the fairness regret incurred is zero. Therefore, the total fairness regret is accumulated across all agents up to time \(T\), reflecting the extent to which the agents' reward under policy $\pi^t$ deviates from their minimum guarantees.

\subsection{Social Welfare Maximization with Minimum-reward Fairness Guarantee}
\label{sec:regret}

Let $\textsc{SW}_{\pi}(T) := \sum_{t=1}^T \sum_{ i\in [n]} \langle A_i, \pi^t \rangle$ represent the total expected social welfare achieved by the policy $\pi = (\pi^t)_{t \geq 1}$ over the time horizon,  $T$. 
\begin{figure}[H]
    \centering
\textbf{P1}
\begin{align}
\text{Maximize}_{ \pi = (\pi_1, \pi_2, \cdots, \pi_T) } \quad & \textsc{SW}_\pi(T)  \label{eq:POne} \\
\text{subject to} \quad & A \pi^t  \geq C \cdot  \Amax \ \ \  \forall t \in [T] \nonumber 
\end{align}
\Description{The original optimization problem.}
\end{figure}
It is easy to see that the optimal fair policy $\pi^* $ pulls each arm with the same probability in each round, i.e., $\pi^*_i = \pi^t_i$ for all $t$ since  matrices $A$ and $C$ are fixed.  

We further assume that the conditions in Theorem~\ref{thm:characterization} are satisfied, i.e., $\pi^*$ is well defined. The reward regret is defined as the cumulative loss in social welfare by not following the policy  $\pi^*$ at each time instant.

\paragraph{Connection with Nash Social Welfare:} The Nash Social Welfare (NSW) objective is known to satisfy fairness guarantees in resource allocation scenarios (see \cite{Hossain2020FairAF, caragiannis2019unreasonable} and references therein). However, it falls short in accommodating user-defined fairness requirements. Additionally, the primary aim of NSW is not to maximize social welfare, which is the central objective of our work. Interestingly, our proposed formulation reveals an equivalence with the Nash product. Specifically, 

\begin{align*}
P1 \equiv \argmax_{\pi\in\Delta_m}\ \underbrace{\underbrace{\Pi_{i=1}^n\left( e^{\langle A_i, \pi\rangle} \right)}_{\textup{Nash product for rewards}} \underbrace{\Pi_{i=1}^n\left( \mathbbm{1}_{\langle A_i,\pi\rangle-C_i\Amax_i\geq 0} \right) }_{\textup{Nash product for fairness}}}_{\textup{Nash product for rewards and fairness}}.
\end{align*}

In this formulation, the Nash product for fairness reaches its maximum value of one only when the fairness guarantees are met for all agents. Provided that a feasible policy exists, the Nash product for rewards in our formulation can be interpreted as an NSW configuration, where the agents' rewards are exponentiated. 

Next,  we define the social welfare regret as the additional loss incurred by the algorithm as compared to the optimal fair policy in hindsight.

\begin{definition}[Social Welfare Regret]
Let  $\pi^*$ be an optimal policy (solution of problem P1) for a given MA-MAB instance $\mathcal{I}$.  Further, let  $\pi = (\pi^t)_{t\geq 1}$ be an arm pull strategy. The social welfare regret of $\pi$ on instance $\mathcal{I}$ over time horizon $T$ is defined as  
\begin{equation}
\mathcal{R}_\textsc{SW}^\pi (T) =  T \cdot  \textsc{SW}(\pi^*) - \sum_{t=1}^T \textsc{SW}(\pi^t)  
\end{equation}
\end{definition}

We drop the superscript in the notation of fairness regret and social welfare regret  whenever the arm-pull strategy  $\pi$ is clear from the context. Note that the expected cumulative SW regret could well  be negative, in which case the policy $\pi$ generates more social welfare than an optimal fair policy at the cost of fairness regret. 

%% file: AAMAS25/results.tex
Throughout the paper, we will consider that the reward functions $X_{i,j}$'s are  sub-gaussian  random variables with finite and positive mean. 

\begin{definition}[Sub-gaussian Rewards]
 We call $X$  a sub-gaussian random variable if  there is a positive constant $\sigma$ such that for every $\lambda \in \mathbb{R}$,  we have 
\begin{equation}
    \mathbb{E}\left[\exp{\left(\lambda(X-\mathbb{E}[X])\right)}\right]\leq \exp(\lambda^2\sigma^2/2).
\end{equation}
\end{definition}

Sub-gaussian random variables encompass a diverse range of distributions, including Bernoulli random variables. More generally, any random variable bounded in $[a, b]$ is $\sigma$-sub-gaussian with $\sigma=\frac{(b-a)}{2}$. The sub-gaussian property ensures that the probability of extreme reward values is minimized, which contributes to better reward guarantees. This characteristic is particularly advantageous for designing and analyzing learning algorithms, allowing us to consider a more general class of reward distributions.

\section{Warmup:  Two Arms Case}\label{sec:warmup}

In this section, we consider a simple MA-MAB setup with 2 arms and $n$ agents. This setup allows us to write the optimal fair policy in tractable mathematical form. We also provide our first algorithm, {\sc Explore-First}.

Consider the MA-MAB instance with $2$ arms and $n $ agents.
Index the agents such that the first $ n_1$ agents prefer arm 1 and the next $n - n_1$ agents prefer arm 2. Note that when $n_1 = n$ (when $n_1 = 0$), we have that all the agents prefer arm 1 (arm 2) and the optimal fair policy, in this case, is straightforward: pull arm $1$ (or arm 2 respectively) with probability 1. Hence, without loss of generality, let $0 <   n_1 < n$. That is,  $ \Amax_i  = A_{i,1} (\geq A_{i,2})$ for all $i \in [n_1]$, $ \Amax_i  = A_{i,2} (> A_{i,1})$ for all $i \notin [n_1]$.  Further assume without loss of generality that arm $1$ is an optimal arm i.e. $\sum_{i \in [n]} A_{i,1} \geq \sum_{i \in [n]} A_{i,2} $ and let $[x^*, 1-x^*]$ be the optimal arm pulling policy. 

To characterize the optimal arm pulling strategy in the two-arms case,  first observe the following property of the optimal fair policy. An optimal fair policy pulls a sub-optimal arm (arm 2) with a nonzero  probability i.e. $1-x^* > 0$ only  when the minimum reward  fairness guarantee  is violated for some agent.   
With this intuitive understanding,  we now characterize the optimal policy $[x^*, 1-x^*]$. 

Let $\Delta:=\sum_{i=1}^n (A_{i,1} - A_{i,2}) >0$. The regret of the policy $\pi$  can be written in terms of $\Delta$ as follows  
\begin{equation}
    \mathcal{R}_\textsc{SW}(T) = \sum_{t=1}^T [x^* - x^t] \Delta. 
    \label{eq:eqnClosedFormSW}
\end{equation}
Here, $x^t$ is the probability of pulling arm $1$ at time $t$.

\begin{restatable}{lemma}{propOne} The optimal feasible policy of a  fair MA-MAB instance  $\mathcal{I}$ with two arms is given by 
\label{prop:One}
    \begin{equation}
    x^* = \min  \Bigg ( 1, \min_{i \in [n] \setminus [n_1]}\frac{ 1- C_i}{ 1- \frac{A_{i,1}}{A_{i,2}}} \Bigg ).
    \end{equation}
\end{restatable}


The proof of Lemma~\ref{prop:One} is given in the Appendix. We are now ready to present our first algorithm that achieves a  sublinear regret guarantee. 

\begin{algorithm}
\caption{\textsc{Explore-First}}
\label{algOne-EF}
\begin{algorithmic}[1]
\STATE \textbf{Require:} $T, C$.
\FOR{$t = 1,2, \cdots, \lfloor T^{\alpha} \rfloor $}  
\STATE Pull arm $i = t\mod(2)+1$. 
\ENDFOR 
 \STATE Compute the estimated reward matrix $\widehat{A}$ of the rewards observed so far.
\STATE Compute $x'=\min \Bigg ( 1, \min_{ i: \widehat{A}_{i,2} > \widehat{A}_{i,1} }\frac{1 - C_i}{1 - \frac{\widehat{A}_{i,1}}{\widehat{A}_{i,2}} }\Bigg ).$
\FOR{$t = \lfloor T^{\alpha} \rfloor + 1, \lfloor T^{\alpha} \rfloor + 2, \cdots,  T $}  
\STATE  Pull arm $1$ with probability $x'$   and   arm $2$ 
with probability $1-x'$.     
\ENDFOR 
 \end{algorithmic}
\end{algorithm}

 \subsection{ Regret Analysis of Explore-First Algorithm }

The \textsc{Explore-First} algorithm addresses the exploration-exploitation tradeoff effectively by delineating the exploration phase from the exploitation phase. During the exploration phase, the algorithm employs a round-robin strategy to pull each arm for $\lfloor T^{\alpha} \rfloor$ rounds. This approach ensures that each arm is sampled sufficiently, yielding more accurate estimates of each arm's reward. However, this phase does not prioritize the arm with the highest reward and does not guarantee immediate rewards for the agents.

In the subsequent exploitation phase, the algorithm utilizes these reward estimates to solve an optimization problem P1. For the specific case of two arms, a closed-form solution is given in Line 6. The optimal fair policy derived from this solution is then used to determine the arm pulls for the remaining rounds.

It is important to note that the regrets associated with both social welfare and fairness are influenced by the choice of the parameter $\alpha$. Specifically, the regret incurred during the exploration phase is proportional to $T^{\alpha}$ in both cases. Thus, a larger value of $\alpha$ results in higher regret due to the increased duration of the exploration phase. Conversely, if $\alpha$ is too small, the estimates of the arm rewards may not be sufficiently accurate, leading to suboptimal decisions in the exploitation phase and, consequently, higher regret. This tradeoff highlights the importance of carefully choosing $\alpha$ to obtain a  balance between accurate reward estimation and minimizing regret.

\noindent

\begin{restatable}{theorem}{ExploreFirstFairnessRegret} [Informal]
\label{thm:ExploreFirstRegret}
The  \textsc{Explore-First} algorithm achieves, 

\begin{enumerate}
    \item expected social welfare  regret of $O \Big(\frac{n}{a_{\min}} T^{2/3} \sqrt{\log(T)} \Big)  $,  and 
    \item expected fairness regret of $O \Big(\frac{n}{a_{\min}} T^{2/3} \sqrt{\log(T)} \Big)  $,\\
    where $a_{\min}=\min_{i,j}A_{i,j}>0$. 
\end{enumerate} 
\end{restatable}
Detailed proof of Theorem~\ref{thm:ExploreFirstRegret} is given in the Appendix. It is easy to see that the \textsc{Explore-First} algorithm is inadequate for both fairness and social welfare. Firstly, observe that the algorithm fails to collect information gathered during the exploit phase and, thus, ceases learning after the exploration phase. This impacts both social welfare and fairness regret guarantees, as inaccurate estimates can result in suboptimal policies. Next, we propose a UCB-based policy that provides a better tradeoff in terms of social welfare regret and fairness regret. 

\section{ The Proposed Algorithm and Analysis}\label{sec:UCB}
\gan{Write the algorithm details}
At each time step $t$, our proposed algorithm \ouralgo\  (refer to Algorithm~\ref{algUCB}) keeps an Upper Confidence Bound (UCB) estimate and a Lower Confidence Bound (LCB) for every arm-agent pair  $(i,j)$. During the initial $t'$ rounds (Lines 2-7), the \ouralgo\  performs exploration, i.e. pulls the arms in a round-robin manner. In the following exploitation phase (i.e., $t \geq t'$), the algorithm keeps UCB and LCB estimates for each arm-agent combination. The UCB index is utilized to provide an optimistic estimate of social welfare, while both UCB and LCB indices are used to assess the fairness requirements to determine the arm-pulling strategy as given in problem P2 below.
\begin{figure}[ht!]
    \centering
\centering 
\textbf{P2}
\begin{align}\label{P2}
\text{Maximize}_{ \pi \in \Delta_{m} } \quad & \sum_{i=1}^n \langle \overline{A}_i, \pi\rangle \\
\text{subject to} \quad & \overline{A} \pi  \geq C \cdot \underline{\Amax} \nonumber
\end{align}
\Description{Our reformulated UCB-LCB-based optimization problem.}
\end{figure}

\piyushi{May be we should (re)-write what $\overline{A}, \underline{A}$ means.} While using the UCB index to estimate rewards is common in literature and used in virtually all UCB-based algorithms, the use of LCB to estimate fairness constraints is not common.  We employ the LCB estimate to ease the fairness constraints in P2, ensuring that the below two properties hold with high probability,
\begin{enumerate}
    \item   the social welfare guarantees remain intact,  and   
    \item  the fairness constraints are met.
\end{enumerate}  
Our proof crucially uses the above two properties of the solution obtained by solving the linear program P2. In particular,  we show the optimal solutions of P2 exhibit similar  social welfare  with a  small loss in fairness guarantee   in comparison with the solution of P1. 

We  begin our analysis with a standard result in probability theory. 
\begin{lemma}[Hoeffding’s inequality for sub-Gaussian random variables]\label{lem:hoeffding}
   Let \( Z_1, Z_2, \dots, Z_k \) be independent sub-Gaussian random variables, each with sub-Gaussian parameter \( \sigma \) and let $S_k = \frac{1}{k} \sum_{s=1}^k Z_s$.  Then  for all \( \varepsilon > 0 \), we have 
\[
P\left( \left| S_k - \mathbb{E}[S_k] \right| > \varepsilon \right) \leq 2 \exp \left( - \frac{ k \varepsilon^2}{2  \sigma^2} \right).
\]
Alternatively, for any \( \delta \in (0, 1] \), with probability at least \( 1 - \delta \),
\[
\left| S_k - \mathbb{E}[S_k] \right| \leq \sigma \sqrt{\frac{2 \log\left( \frac{2}{\delta} \right)}{k}}. 
\]
 
\end{lemma}

\noindent We now prove an important technical lemma.

\begin{restatable}{lemma}{ImpLem} Let $\pi^*$ be an optimal feasible solution of P1 and for any $t \geq t'$,   $\pi^t$ be an optimal solution of  P2 with $\overline{A}:= \overline{A}^t$. Then with probability at-least $1-  1/\sqrt{T}$, we have  \[  SW_{\pi^t} (\overline{A}) \geq SW_{\pi^*}( A ).  \]
\label{lem:equivalence}
\end{restatable}

\begin{proof}
Let $\varepsilon_{i,j}^t = \sigma \sqrt{\frac{2 \log(4mn\sqrt{T})}{N_j^t}} $ (See Line 9 of \ouralgo\ algorithm) and define   
 $ \delta'_{i,j} := \exp\left(\frac{-N_j^t \varepsilon^2}{2\sigma^2}\right)  = \frac{1}{4mn\sqrt{T}}$. Further,  let $\delta = 1/\sqrt{T}$.  Note that $\delta'_{i,j}$ is the  probability that $\overline{A}_{i,j}^t = \widehat{A}_{i,j}^t + \varepsilon_{i,j}^t \leq A_{i,j}$ at some time $t \geq t'$. 
 
 By symmetry of tail bounds  around the mean value  given by Hoeffding's inequality (Lemma~\ref{lem:hoeffding}) we have that  $ \delta'_{i,j} $ is also the  probability that $A_{i,j} \geq \underline{A}_{i,j}^t $.    
 
Note that arms are pulled in a round-robin fashion in the exploration phase of the \ouralgo\ algorithm. This implies that   each arm is pulled for the same number of rounds, i.e., $N_j^{t'}$ is the same for all $j \in [m]$.   Hence, we have that $\delta' := \delta'_{i,j} = \frac{1}{4mn\sqrt{T}} $. 
 
 We prove the stated claim by showing that  with probability at least $1- \frac{1}{\sqrt{T}}$,  every feasible policy $\pi$ of P1 is also  a feasible policy of P2. 
 
 Fix $i \in [n]$. Let $k  \in  \arg\max_{j \in [m]} A_{i,j}$ and $ k^t \in   \arg\max_{j \in [m]} \underline{A}_{i,j}^t $ be the least indexed arms with maximum value in the $i^{\textup{th}}$ row of matrices $A$ and $\underline{A}$ respectively. 

 We have   
\begin{align}\label{eqn:sw}
\langle \overline{A}_i,  \pi \rangle &  \underbrace{\geq}_{w.p. \geq 1-m \delta'} \langle  A_i, \pi  \rangle      \geq  C_i \cdot A_{i,k} \geq C_i \cdot  A_{i, k^t} \nonumber  \\  & \underbrace{\geq}_{w.p. \geq 1-m\delta'} C_i \cdot (\widehat{A}_{i,k^t} - \varepsilon_{k^t}^t) = C_i \cdot \underline{A}_{i,k^t}. 
\end{align}
The first inequality (from left) in Equation~\ref{eqn:sw} above  follows from Hoeffding's inequality (Lemma~\ref{lem:hoeffding}), the second inequality follows from the feasibility of $\pi$ for P1, the third inequality follows from the definition of $k^t$ and the last inequality again follows from Hoeffding's inequality (Lemma~\ref{lem:hoeffding}). The first and last inequalities each hold with probability at least $(1-m\delta')$. Hence we have with probability at-least $1-2m\delta'$  that $\langle \overline{A}_i,  \pi \rangle \geq  C_i \cdot \underline{A}_{i,k^t}$.

Using union bound, we have  with probability at-least $1 - 2nm\delta'$ that $\langle \overline{A}_i,  \pi \rangle \geq  C_i \cdot \underline{A}_{i,k^t} $ for all $i\in [n]$; i.e.   every  feasible solution $\pi$ of P1 is also a feasible solution of P2. In particular, $\pi^*$ is a feasible solution of P2 with probability at least $1 - \delta/2$. This, along with the definition of $\pi^t$  gives 
\begin{align*}
SW_{\pi^t} (\overline{A}) \underbrace{\geq}_{w.p. \geq 1- \delta/2 }  SW_{\pi^*} (\overline{A}) \underbrace{\geq}_{w.p. \geq 1- \delta/2 } SW_{\pi^*} (A)
\end{align*}

This proved the stated claim.   
\end{proof}
 Lemma~\ref{lem:equivalence} implies that by easing the fairness constraints, it is possible to increase social welfare.  
 \gan{get this proof from appendix to main...}

\begin{algorithm}[tb]
\caption{\ouralgo}
\label{algUCB}
\begin{algorithmic}[1] 
\STATE \textbf{Require:} $T, n, m , C, N_{j}^t=0 \ \forall j \in [m].$
\STATE $t'=m\lceil \sqrt{T} \rceil, \ t=1$, $\widehat{A}=\mathbf{0}_{m\times n}$.
\FOR{$t\leq t'$}
\STATE Pull arm $j'=t \textup{ mod } m + 1$.
\STATE $\forall i \in [n]$, observe reward $X_{i, j'}^t\sim \mathcal{D}(\mu_{i, j'})$.
\STATE $\forall i \in[n], \forall j \in[m],$ $\widehat{A}_{i,j} = 
    \begin{cases}
     \widehat{A}_{i,j} & \text{ if }  j\neq j'  \\
      \frac{(N_j^{t-1}) \widehat{A}_{i,j} + X^t_{i,j}  }{N_j^t}  & \text{ if } j =j'. \\ 
      \end{cases}$
\STATE $N_{j'}^ t = N_{j'}^ {t-1}+1.$
\ENDFOR
\STATE Compute the confidence matrix $\mathcal{E}$ with entries $$\epsilon_{i,j}^t=\sigma \sqrt{\frac{2\log{(8mnT)}}{N_j^ t}} \ \forall i\in [n], j\in [m].$$ 
\FOR{$t\leq T$}
\STATE Compute $\overline{A} = \widehat{A} + \mathcal{E}$ and $\underline{A} = \widehat{A} - \mathcal{E}$.
\STATE $\forall i \in [n]$, compute $\underline{\Amax}_i = \max_{j\in [m]}\underline{A}_{i,j}$ (break ties arbitrarily).
\STATE Solve \textbf{P2} and 
     let $\pi'$ be the solution of this LP. 
\STATE Sample $j' \sim \pi'$.
\STATE $\forall i \in [n]$, observe reward $X_{i, j'}^t\sim \mathcal{D}(\mu_{i,j'})$.
\STATE $N_{j'}^ t = N_{j'}^ {t-1}+1.$
\STATE  $\forall i \in[n], \forall j \in[m]$, $\widehat{A}_{i,j} = 
    \begin{cases}
     \widehat{A}_{i,j} & \text{ if }  j\neq j'  \\
      \frac{(N_j^{t-1}) \widehat{A}_{i,j} + X_{i,j}^t  }{N_j^t}  & \text{ if } j =j'. \\ 
      \end{cases}$
\STATE Update entries of $\mathcal{E}$.
\ENDFOR
\end{algorithmic}
\end{algorithm}
\himanshu{ Algorithm 2, line 6, its calculating the estimated reward upto time $t' $ but in Algorithm 1, we wrote one line ie line 5. Should we change any one for consistancy}

\begin{restatable}{theorem}{rewardRegretUCB}
\label{thm:rewardRegretUCB}
For any feasible MA-MAB instance $\mathcal{I}$ with   $T\geq 32n^2\sigma^2$,     expected social welfare regret of \ouralgo \ is upper-bounded by $$ 4n \sqrt{2T}\left(\sigma\log{(2m^2T)} +m + \sigma \right).  $$
    
\end{restatable}

\gan{give  outline of the proof with technical details }
The detailed proof of~Theorem \ref{thm:rewardRegretUCB} is provided in the Appendix. We provide a high-level overview of the proof here.  First, we break down the regret into two components: R1 and R2, representing the regret from the exploration phase and exploitation phase, respectively. The component R1 encompasses the regret over the initial exploration phase of  $m\lceil \sqrt{T} \rceil$ rounds. Following this, using      Lemma~\ref{lem:equivalence} we  argue that the social welfare obtained by solving P2 is at-least that of the original problem with high probability. By applying Hoeffding's inequality and the union bound to the aggregated expected values, we show that R2 is capped by $\tilde{O}(T^{1/2})$. 
 We emphasize here that the social welfare regret  of \ouralgo\ is asymptotically optimal (refer to Section~\ref{subsec:lowerBound} for the lower bound). 

We now give the fairness regret guarantee of \ouralgo\ algorithm.

\begin{restatable}{theorem}{fairnessRegretUCB}
\label{thm:fairnesRegretUCB}

For any feasible MA-MAB instance $\mathcal{I}$ with   $T\geq 32n^2\sigma^2$,     expected fairness regret of \ouralgo \ is upper-bounded by
$$ 6n(\textstyle \max_{i\in [n]} C_i)\sigma T^{3/4}\log{(2m^2T)} + mnO(\sqrt{T}). $$
\end{restatable}

It is worth noting that while the social welfare regret guarantee of \ouralgo\ is much stronger, this comes at the cost of higher  fairness  regret. We demonstrate this trade-off in  Section~\ref{sec: simulation} with simulations. 
\gan{Need two paragraphs with proof intuition in the main. Detailed proof may go in the appendix. Most of this part is handwavy ...need to make it more concrete}
Next, we show the lower bound on fairness and social welfare regret guarantees of  MA-MAB problems, 

\section{Regret Lower Bounds}
\label{subsec:lowerBound}
In this section, we  prove that every algorithm must suffer an \emph{instance-independent regret} \footnote{ A regret guarantee is called instance independent if it  holds for every feasible MA-MAB instance  $\mathcal{I}$.  That is,  for all values of fairness constraints matrix $C$, time horizon $T$, and mean rewards matrix $A$ provided that  $\mathcal{I}$ admits a feasible policy.}  of $\Omega(\sqrt{T})$  in both fairness and social welfare.     
\begin{restatable}{theorem}{lowerRegret}
\label{thm:lowerBound}
An instance-independent social welfare and fairness regrets of MA-MAB problem is lower bounded by $(\Omega(\sqrt{T}), \Omega(\sqrt{T}))$. 
\end{restatable}
\begin{proof}[Proof Outline]
The proof of Theorem \ref{thm:lowerBound} is given in the Appendix. 
 We provide an intuition here. To show the lower bound on social welfare, consider a class of instances where each row is a non-negative  multiple of the first row i.e. $A_i = \beta_i A_1$ for some $\beta_i \geq 0$ and $C$ as a zero matrix. As every agent has the same preferences over arms, the problem of maximizing social welfare now is reduced to the problem of identifying an arm $j$ with the highest $\sum_{i \in [n]} A_{i,j}$. This is equivalent to finding an arm  $j$ with largest $\Big( 1 +  \sum_{i\neq 1} \beta_i \Big) \cdot A_{1,j}$. This problem is the same as the classical stochastic MAB problem with $m$ arms and reward distributions with the mean reward of arm $j$ as $\Big( 1 +  \sum_{i\neq 1} \beta_i \Big) \cdot A_{1,j}$.   We use the $\Omega(\sqrt{T})$ instance-independent regret  lower bound \cite[Theorem 5.1]{Auer02}  for  classical stochastic  bandits to lower bound the social welfare regret.    

To lower-bound the fairness regret, we construct a MA-MAB instance with $m=2$  arms and $n=2$ agents with $A$ being the Identity matrix. For any values of $C_1 > 0 $  and $ C_2 =1- C_1$ satisfying the conditions in Theorem~\ref{thm:characterization}, the fairness criteria is satisfied if and only if $ x^* = C_1$. Since $ A_{1,1} - A_{1,2} = - (A_{2,1} - A_{2,2}) = 1$ we have that the fairness regret can be written (using Eq. \ref{eq:eqnClosedFormSW}) as 
$$\mathcal{R}_{\textsc{FR}}(T) = \sum_{t=1}^T |C_1 - x^t|_+ \geq | \sum_{t=1}^T  C_1 -  \sum_{t=1}^T x^t|_+ \geq \sum_{t=1}^T  C_1 -  \sum_{t=1}^T x^t. $$  Now consider a stochastic MAB setup with two arms and rewards $1 $ \piyushi{I think this should be 1} and $0$ respectively. We again use a lower bound of $\Omega(\sqrt{T})$ for the stochastic MAB setting with this instance to obtain a lower bound on the fairness regret for the MA-MAB setting.     
\end{proof}

It is worth noting that the lower bounds presented in Theorem~\ref{thm:lowerBound} invoke different stochastic MAB instances and, therefore, may not hold simultaneously. The \ouralgo\ algorithm proposed in this paper is asymptotically optimal up to a logarithmic factor; however, the same is not true for the fairness regret.   

In the next section, we present a heuristic algorithm that provides better empirical performance for fairness. However, the empirical performance for the social welfare of this heuristic algorithm is worse than that of \ouralgo.

%% file: AAMAS25/simulation.tex
\input{AAMAS25/dual} 
\section{Experimental Evaluation}
\label{sec: simulation}

\begin{figure*}[ht!]
\centering
\begin{subfigure}{.33\textwidth}
    \centering
    \includegraphics[width=\linewidth]{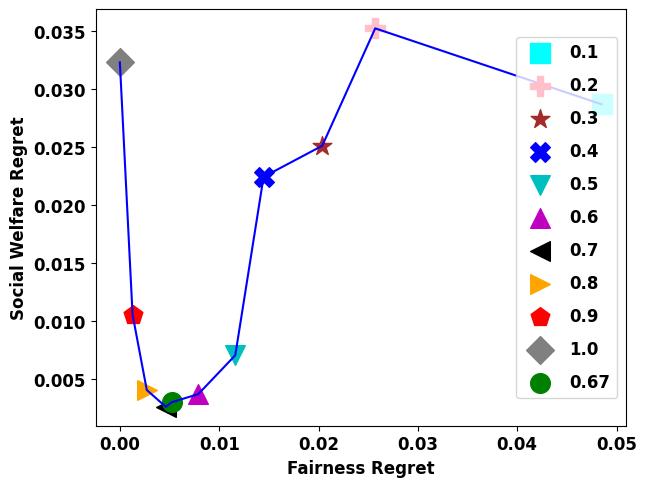}  
    \caption{Explore-Exploit tradeoff with {\sc Explore-First}.
    }
    \label{sim1}
\end{subfigure}
\begin{subfigure}{.33\textwidth}
    \centering
    \includegraphics[width=\linewidth]{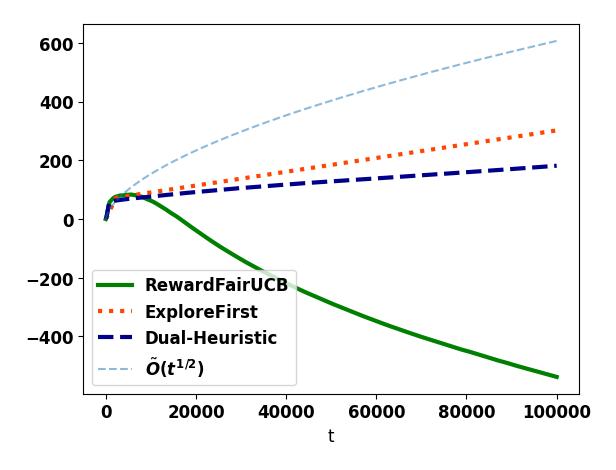}  
    \caption{Social welfare regrets vs timesteps.}
    \label{sim2}
\end{subfigure}
\begin{subfigure}{.33\textwidth}
    \centering
    \includegraphics[width=\linewidth]{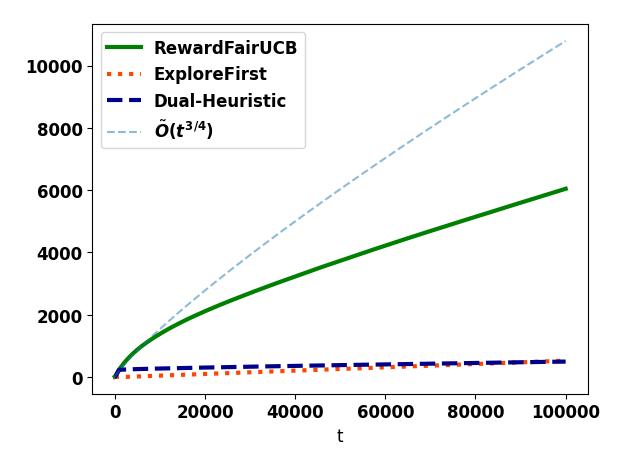}  
    \caption{Fairness regrets vs timesteps.}
    \label{sim3}
\end{subfigure} 
    \caption{Experimental results on simulated data ($n=4, m=3$). $C_i$ is $0.3\ \forall i \in [n]$.}
    \label{fig:rewardfair_ucb-sim}
    \Description{Experimental results on simulated data.}
\end{figure*} 

\begin{figure*}[ht!]
\centering
\begin{subfigure}{.33\textwidth}
    \centering
    \includegraphics[width=\linewidth]{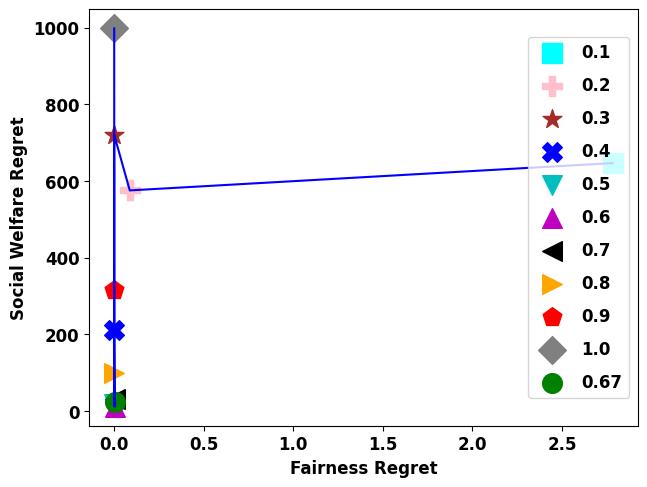}  
    \caption{Explore-Exploit tradeoff with {\sc Explore-First}.
    }
    \label{rw1}
\end{subfigure}
\begin{subfigure}{.33\textwidth}
    \centering
    \includegraphics[width=\linewidth]{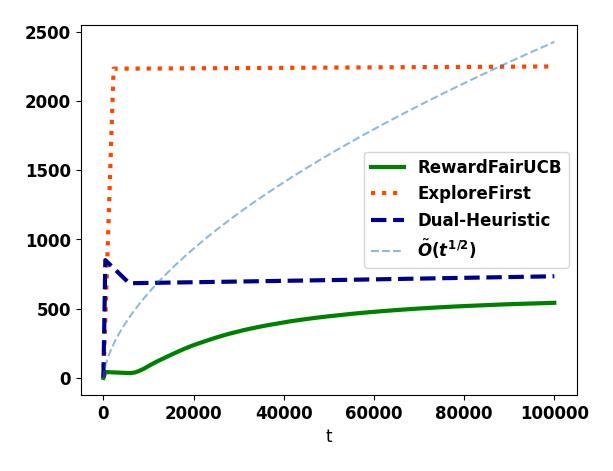}  
    \caption{Social welfare regrets vs timesteps.}
    \label{rw2}
\end{subfigure}
\begin{subfigure}{.33\textwidth}
    \centering
    \includegraphics[width=\linewidth]{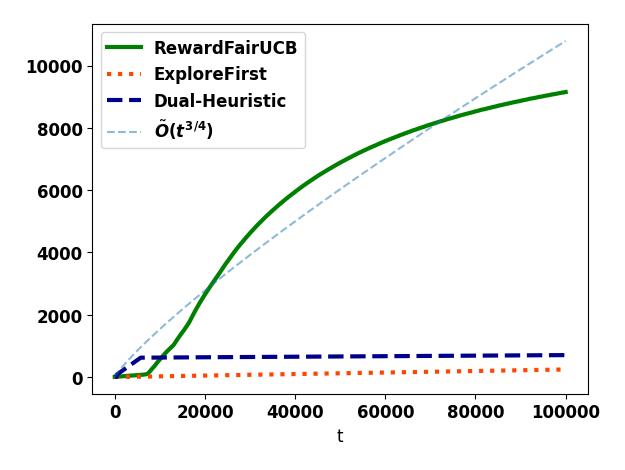}  
    \caption{Fairness regrets vs timesteps.}
    \label{rw3}
\end{subfigure}
    \caption{Experimental results on MovieLens real-world data ($n=6039, m=18$). $C_i$ is $1/m \ \forall i \in [n]$.}
    \label{fig:rewardfair_ucb}
    \Description{Experimental results on MovieLens data.}
\end{figure*} 


We now empirically validate the sublinear regret guarantees of the proposed algorithms and the efficacy of \ouralgo. 
\subsection{Common Experimental Setup}
The distribution $\mathcal{D}(\mu_{i,j})$ is taken $\textup{Ber}(\mu_{i, j})$ i.e. when an arm $j'\in [m]$ is sampled, agent $i\in [n]$ obtains a reward drawn from Bernoulli with parameter $\mu_{i, j}$. We plot the average regrets after simulating with 100 runs for different realizations of randomly sampled rewards. Complementing our analysis for \textsc{Explore-First} in the 2-arm case, we empirically find that the same regret guarantees for $m>2$ hold with same optimal exploration parameter valued, i.e.,  $\alpha$ = 0.67. For this, we replace Step (6) of \textsc{Explore-First} algorithm with the solution of the linear program P1 obtained with empirical reward estimates $\hat{A}$. The CVXPY \citep{diamond2016cvxpy} library is used to solve the linear programs wherever required in our algorithms. More results with different $A$ matrices are presented in the Appendix.
\subsection{Experiments on Simulated Data} 
Figure~\ref{fig:rewardfair_ucb-sim} shows results with a mean reward matrix $A$ of size (4, 3), i.e. $n=4, \ m=3$. The stopping time $T=10^5$ and $C_i=c=0.3\ \forall i\in [n]$. We first empirically show the trade-off between exploration and exploitation by plotting the social welfare regret and fairness regret of \textsc{Explore-First} algorithm on varying the exploration parameter $\alpha$. Figure~\ref{sim1} shows the plots comparing the social welfare regret and the fairness regret on varying the exploration parameter $\alpha$ from $\{0.1, 0.2, \cdots, 1.0, 0.67\}$, marked in the legend. The plotted regret values are after normalizing by $T$. The empirically observed best choice for obtaining low regrets for both social welfare and fairness is $\alpha=0.67$ which closely matches the theoretically optimal value of $\alpha=2/3$ derived for the 2-arm case in Sec~\ref{sec:warmup}. 

Figures~\ref{sim2} and~\ref{sim3} respectively compare the social welfare regrets and the fairness regrets of \ouralgo \ with the \textsc{Explore-First} baseline (with $\alpha=0.67$) and the dual heuristic (Sec~\ref{sec:dual}). Figure~\ref{sim2} shows that \ouralgo\ not only obtains a sub-linear regret but also outperforms the baselines and heuristics. We can also see sublinear regrets obtained by \textsc{Explore-First}, supporting our theoretical claim derived for the 2-arm case. Figure~\ref{sim3} demonstrates sublinear fairness regret of \ouralgo. While the \textsc{Explore-First} baseline and the dual-based heuristic obtain a lower fairness regret, they incur an excess social welfare regret. \ouralgo\ achieves optimal social welfare performance while maintaining a sublinear fairness regret.

\subsection{Experiments on Real-World Data}
Figure~\ref{fig:rewardfair_ucb} shows the performance of our algorithm on real-world data, MovieLens 1M \cite{10.1145/2827872}. MovieLens comprises ratings given by users to different movies. We obtain a user-genre matrix with the average rating that users assign to each movie genre. This matrix is normalized to have each entry in $[0, 1]$ and serves as the mean reward matrix $A$. For the movies associated with multiple genres, their contribution to each genre was divided equally. 

The $A$ matrix for this experiment is of size (6039, 18), i.e. $n=6039, \ m=18$. The stopping time $T=10^5$ and $C_i=c=1/m\ \forall i\in [n]$. We first empirically show the trade-off between social welfare regret and fairness regret of \textsc{Explore-First} algorithm on varying the exploration parameter $\alpha$.

We begin by empirically illustrating the effect of exploration and exploitation trade-off, controlled by the exploration parameter $\alpha$, on the social welfare regret and fairness regret of the \textsc{Explore-First}. Figure~\ref{rw1} shows the two regrets (normalized by $T$) with $\alpha$ values ranging from $\{0.1, 0.2, \cdots , 1.0, 0.67\}$ as marked in the legend. The empirically determined optimal that minimizes both social welfare regret and the fairness regret is $\alpha=0.67$ which closely matches the theoretically optimal $\alpha=2/3$ derived for the 2-arm case in Sec~\ref{sec:warmup}. Figures~\ref{rw2} and~\ref{rw3} compare the social welfare and fairness regrets of \ouralgo\ with the baseline algorithm \textsc{Explore-First} and the dual-heuristic. Figures~\ref{rw2} and~\ref{rw3} empirically demonstrate that \ouralgo\ obtains a sublinear regret for both social welfare and fairness. Although the \textsc{Explore-First} baseline and the dual-based heuristic achieve lower fairness regret, they lead to a higher social welfare regret. \ouralgo\ performs optimally in terms of social welfare and still obtains a sublinear fairness regret. The empirical results with \textsc{Explore-First} also support our theoretically sublinear regret claim that was derived for the 2-arm case.

%% file: AAMAS25/dual.tex
\section{Dual-Based Algorithm}\label{sec:dual}
We now present a heuristic dual-based algorithm inspired by the dual formulation of our optimization problem P1 (Eq. \ref{eq:POne}). This algorithm pulls each arm in  round-robin fashion for the first $O(\sqrt{T})$ rounds  and solves the problem using a Lagrangian dual formulation with the appropriate  estimates of $A$ using pulls from the  first phase. 

We start  with formulating the dual of our Problem P1. 
\begin{align}
    \textup{\textbf{Primal:}}& -\min_{\pi\in\Delta_m} \sum_{i=1}^n -\langle A_i, \pi\rangle \nonumber
    \\
    &\textup{s.t. }-(\langle A_i, \pi\rangle - C_i\Amax_i)\leq 0 \ \forall i\in[n].
\end{align}
We derive the Lagrangian dual with Lagrange parameters $\lambda_i|_{i=1}^n$ corresponding to the fairness constraints.
\begin{align}
    \textup{\textbf{Dual:}}\quad &-\max_{\lambda \in \mathbb{R}^n \geq 0}\ \min_{\pi\in \Delta_m} \sum_{i=1}^n-\langle A_i, \pi\rangle - \sum_{i=1}^n \lambda_i(\langle A_i, \pi\rangle - C_i\Amax_i) \nonumber\\
    =& -\max_{\lambda \in \mathbb{R}^n \geq 0}\ \min_{\pi\in \Delta_m} -\left\langle\sum_{i=1}^n(1 +\lambda_i)A_i, \pi\right\rangle + \sum_{i=1}^n \lambda_i C_i\Amax_i \nonumber\\
    =& -\max_{\lambda \in \mathbb{R}^n \geq 0} \ -\max_{j\in[m]}\left(\sum_{i=1}^n(1+\lambda_i)A_i\right)_j + C_i\langle\lambda, \Amax \rangle\nonumber\\
    =&-\max_{\lambda \in \mathbb{R}^n \geq 0} \ - \|\left(\textup{Diag}(1+\lambda)A\right)^\top \mathbf{1}_n\|_\infty + C_i\langle\lambda, \Amax \rangle \label{dualcp}
\end{align}
where $\textup{Diag}(\cdot)$ denotes the diagonal matrix formed by the entries $(1+\lambda_i)$.
Motivated by the simplification we obtain in the last step, our dual algorithm is designed to pick the arms based on the UCB estimate of $\|\left(\textup{Diag}(1+\lambda)A\right)^\top \mathbf{1}_n\|_\infty$ with $\lambda$ as the solution of Eq.~(\ref{dualcp}).

\begin{algorithm}
\caption{\textsc{Dual-Inspired Algorithm.}}
\label{alg-dual}
\begin{algorithmic}[1]
\STATE \textbf{Require:} $T, n, m , C, N_{j}^{0}=0 \ \forall j \in [m]$.
\STATE $t'=m\lceil \sqrt{T} \rceil, \ t=1$,  $\widehat{A}=\mathbf{0}_{m\times n}$.
\FOR{$t\leq t'$}
\STATE Pull arm $j'=t \textup{ mod } m + 1$.
\STATE $\forall i \in [n]$, observe reward $X_{i, j'}^{t}\sim \mathcal{D}(\mu_{i,j'})$.
\STATE $\forall i \in[n], \forall j \in[m]$, $\widehat{A}_{i,j} = 
    \begin{cases}
     \widehat{A}_{i,j} & \text{ if }  j\neq j'  \\
      \frac{(N_{j}^{t-1}) \widehat{A}_{i,j} + X_{i,j}^t}{N_{j}^{t}}  & \text{ if } j =j'. \\ 
      \end{cases}$
\STATE $N_{j'}^ t = N_{j'}^{t-1}+1.$
\ENDFOR
\STATE Compute $\mathcal{E}$ with entries $\epsilon_{i,j}^t=\sigma\sqrt{\frac{2\log{(8mnT)}}{N_{j}^{t}}}.$
\STATE Compute $\hat{\lambda}$ by solving the convex program in Eq. (\ref{dualcp}) with $\hat{A}$.
\FOR{$t\leq T$}
\STATE $j'\in \argmax \left( \left(\textup{Diag}(1+\hat{\lambda})\hat{A}\right)^\top \mathbf{1}_n + \mathcal{E}\right)$.
\STATE $\forall i \in [n]$, observe reward $X_{i, j'}^t\sim \mathcal{D}(\mu_{i,j'})$.
\STATE $N_{j'}^{t} = N_{j'}^{t-1}+1.$
\STATE  $\forall i \in[n], \forall j \in[m],$
$\widehat{A}_{i,j} = 
    \begin{cases}
     \widehat{A}_{i,j} & \text{ if }  j\neq j'  \\
      \frac{(N_{j}^{t-1}) \widehat{A}_{i,j} + X_{i,j}^t}{N_{j}^{t}}  & \text{ if } j =j'. \\ 
      \end{cases}$
\STATE \STATE Update entries of $\mathcal{E}$.
\ENDFOR

 \end{algorithmic}
\end{algorithm}

%% file: AAMAS25/discussion.tex
\section{Discussion and Future Work}
Our paper formulates a fair MA-MAB problem where the search is over reward-based social welfare maximizing policy that also ensures fairness to each agent.
Our notion of fairness guarantees a pre-specified fraction of the corresponding maximum possible rewards to each agent. 
We derive the lower bound of $\tilde{O}(\sqrt{T})$ that holds individually for both social welfare and fairness regret. 
Our proposed algorithm \ouralgo \ obtains an optimal (up to logarithmic constants) social welfare regret and a sub-linear fairness regret. We also propose baseline algorithms/heuristics for the problem, present the exploration-exploitation trade-off and empirically validate the efficacy of the proposed \ouralgo \ algorithm on both simulated and real-world data. Our algorithms can be easily made time-horizon unaware with a doubling trick \cite[Theorem 4]{besson2018doubling}.
Improving the fairness regret upper bound of $\tilde{O}(T^{3/4})$ to match the lower bound of $\Omega(\sqrt{T})$ would be an interesting future work which would also include theoretically analysing the proposed dual-based heuristics.
Another future work could be to extend the lower bounds for the regrets derived individually to hold simultaneously. It will also be interesting to extend our theoretical analysis for \textsc{Explore-First} algorithm (Algo~\ref{algOne-EF}) for $m>2$.

\begin{acks}
PM thanks Google for the PhD Fellowship. GG thanks support from SERB through grant CRG/2022/007927.  The authors also thank the anonymous reviewers for their constructive feedback.
\end{acks}

%% file: AAMAS25/Appendix.tex
\onecolumn
\appendix 
\section{Missing Proofs From Sections~\ref{sec:settings_prelim} and~\ref{sec:warmup}}
\characterization* 
\begin{proof}
\noindent 
We prove the existence of fair policy in two cases separately. 
\newline
\noindent \textbf{Case 1 ($\sum_{i \in [n]} C_i \leq 1$):}   Consider without loss of generality that $C_i > 0$ for at-least one agent $i$.\footnote{For otherwise  if $\sum_{i}C_i =0$,  then for every policy $\pi \in \Delta_m$ we have $A_i\pi \geq 0$ for all $i \in [n]$ and hence every policy is feasible. 
} In particular, $\sum_{i} C_i > 0$. Let $j_i$ denote the smallest index $j \in [m]$ such that $j  \in \arg\max_{j' \in [m]} A_{i,j'}$. Consider the following policy. 
\begin{equation}
    \pi_j = \frac{\sum_{i=1}^n \mathds{1}(j = j_i) C_i}{\sum_{i=1}^n C_i}
    \label{eq:feasiblePolicy}
\end{equation}
First, observe that the policy given in Equation~\ref{eq:feasiblePolicy} is feasible i.e. is a valid probability distribution. It is easy to see that $\pi_j \geq  0$ for all $j$ and further 
\begin{align*}
    \sum_{j=1}^m \pi_j = \frac{\sum_{j=1}^m \sum_{i=1}^n \mathds{1}(j = j_i) C_i}{\sum_{i = 1}^n C_i} = \frac{ \sum_{i=1}^n C_i \sum_{j=1}^m \mathds{1}(j = j_i)}{\sum_{i = 1}^n C_i} = 1.   
\end{align*}
We now show that the policy obtained in Equation~\ref{eq:feasiblePolicy} satisfies fairness constraints. 
\begin{align*}
    A_i \pi  &= \sum_{j=1}^m A_{i,j}\pi_j =  A_{i,j_i}\pi_{j_i} + \sum_{j \neq j_i}^m A_{i,j}\pi_j   \geq  A_{i,j_i}\frac{C_{i}}{\sum_{i=1}^n C_i} \stackrel{(1)}{\geq} A_{i j_i} C_i =  C_i \Amax_i,     
\end{align*}
where inequality (1) above uses that $\sum_{i\in [n]}C_i\in (0, 1]$.
\newline
\textbf{Case 2 ($C_{\max} \leq \frac{1}{\min(n,m)}$): }    If $n \leq m $ then we have $ C_{\max} \leq \frac{1}{n}$ which implies that $\sum_{i=1}^n C_i \leq 1$. This is exactly the Case 1 discussed above. 

For $m < n$,  consider a policy $\pi = [1/m , 1/m, \cdots , 1/m]$ and note that 
\begin{align*}
    A_i \pi = 1/m \sum_{j =1}^m A_{i,j} \geq \frac{A_{i, j_i }}{m}  \stackrel{(2)}{\geq} C_{\max} A_{i, j_i } \geq C_i A_{i, j_i } = C_i \Amax_i,    
\end{align*}
where inequality (2) above uses that $C_{\max}\leq \frac{1}{m}$.
This completes the proof of the lemma. 
 \end{proof} 

\propOne*

\begin{proof}
Let us write the fairness constraints explicitly. Recall from the definition of $n_1$ that 
 for all $i \in [n_1]$  we have 
\begin{align*}
A_{i,1}x + (1-x) A_{i,2} \geq C_i    A_{i,1}   
\iff   x \geq \frac{C_i - \frac{A_{i,2}}{A_{i,1}}}{ 1 - \frac{A_{i,2}}{A_{i,1}}}. 
\end{align*}
Similarly for $i \in [n] \setminus [n_1]$ we have 
\begin{align*}
A_{i,1}x + (1-x) A_{i,2} & \geq C_i    A_{i,2}  
\iff   x  \leq \frac{1 - C_i}{ 1 - \frac{A_{i,1}}{A_{i,2}}}
\end{align*}
A policy $[x, 1-x]$ is feasible iff $\exists x \in [0,1]$ such that $ $ 

$$ \max \Bigg( 0,  \max_{i \in [n_1]}\frac{C_i - \frac{A_{i,2}}{A_{i,1}}}{ 1 - \frac{A_{i,2}}{A_{i,1}}} \Bigg) \leq x \leq \min \Bigg( 1,  \min_{i \in [n] \setminus [n_1]}\frac{1 - C_i}{ 1 - \frac{A_{i,1}}{A_{i,2}}} \Bigg). $$  The optimality of the policy $x^* = \min \Bigg( 1,  \min_{i \in [n] \setminus [n_1]}\frac{1 - C_i}{ 1 - \frac{A_{i,1}}{A_{i,2}}} \Bigg)$ follows from the fact that arm 1 is the optimal arm.  
\end{proof}

\section{Fairness and Social Welfare regret guarantees of \textsc{Explore-First} Algorithm}

Let $\mathcal{I}$ be a MA-MAB instance with $n \geq 1 $ agents and 2 arms. We begin with social welfare regret guarantee for \textsc{Explore-First}. We will first prove useful supporting results. We begin by stating  Chernoff's tail bounds, a well known result in probability theory. 

\begin{lemma}\label{lemma:chernoff-ltb}
[Chernoff's Lower Tail Bound] Let $X = \frac{1}{k}\sum_{i=1}^k X_i$ be the average of $k$ I.I.D. non-negative $\sigma$-sub-Gaussian random variables and let $ \mu=\mathbb{E}[X_i]$. Then for any $\delta\in [0, 1]$,
$$\mathbb{P}(X\leq (1-\delta) \mu)\leq \exp{\left(-\left(\frac{ \mu}{\sigma^2}\right)k\delta^2 \mu/2\right)}.$$
\end{lemma}

\begin{proof}
For any $\lambda\in \mathbb{R}; ~\lambda<0$,
\begin{align*}
\mathbb{P}(X\leq (1-\delta) \mu) = \mathbb{P}(e^{\lambda X}\geq e^{(1-\delta)\lambda \mu}) \stackrel{(M.I.)}{=} \frac{\mathbb{E}[e^{\lambda (X- \mu)}]}{e^{-\delta\lambda \mu}}\stackrel{(I.I.D.)}{=}\Pi_{i=1}^k \frac{\mathbb{E}[e^{\lambda/k(X_i- \mu)}]}{e^{-\delta\lambda \mu}}\leq  \frac{\Pi_{i=1}^s e^{\frac{\sigma^2 \lambda^2}{2k^2}}}{e^{-\delta\lambda \mu}}=e^{\frac{\sigma^2\lambda^2}{2k}+\delta\lambda \mu},
\end{align*}
where the second equality uses the Markov's inequality and the last inequality follows from the definition of $\sigma$-sub-Gaussian random variables. On choosing $\lambda = -\frac{k\delta  \mu}{\sigma^2}$ that gives the tightest upper-bound on the RHS, we obtain the final result.
\end{proof}

\begin{lemma}\label{lemma:chernoff-utb}
[Chernoff's Upper Tail Bound] Let $X = \frac{1}{k}\sum_{i=1}^k X_i$ be the average of $s$ I.I.D. non-negative $\sigma$-sub-Gaussian random variables and let $ \mu=\mathbb{E}[X_i]$. Then for any $\delta\in [0, 1]$,
$$\mathbb{P}(X\geq (1+\delta) \mu)\leq \exp{\left(-\left(\frac{ \mu}{\sigma^2}\right)k\delta^2 \mu/2\right)}.$$
\end{lemma}

\begin{proof}
For any $\lambda\in \mathbb{R}; ~\lambda>0$,
\begin{align*}
    \mathbb{P}(X\geq (1+\delta) \mu) = \mathbb{P}(e^{\lambda X}\geq e^{(1+\delta)\lambda \mu}) \stackrel{(M.I.)}{=} \frac{\mathbb{E}[e^{\lambda (X- \mu)}]}{e^{\delta\lambda \mu}}\stackrel{(I.I.D.)}{=}\Pi_{i=1}^k \frac{\mathbb{E}[e^{\lambda/k(X_i- \mu)}]}{e^{\delta\lambda \mu}}\leq  \frac{\Pi_{i=1}^s e^{\frac{\sigma^2 \lambda^2}{2k^2}}}{e^{\delta\lambda \mu}}=e^{\frac{\sigma^2\lambda^2}{2k}-\delta\lambda \mu},
\end{align*}
where the second equality uses the Markov's inequality and the last inequality follows from the definition of $\sigma$-sub-Gaussian random variables. On choosing $\lambda = \frac{k\delta  \mu}{\sigma^2}$ that gives the tightest upper-bound on the RHS, we obtain the final result.
\end{proof}
We also prove another useful lemma which will be used in bounding the regrets of \textsc{Explore-First}. We present the proofs with $C_i=c\ \forall i\in [n]$. We will use that both arms are equally pulled up to $2t'$. We consider $[n]\setminus [n_1]$ to be non-empty as otherwise, we have a trivial optimal policy. 
\begin{restatable}{lemma}{lemOne}
\label{lem:lemOne}
Let  $a_{\min} : = \min_{i \in[n],j \in [m]} A_{i,j} > 0$ and $x^t$ be the \textsc{Explore-First} policy.  Then,  with  probability atleast $ 1 - 2 \exp(-t'\delta^2 a_{\min}^2/(2\sigma^2))$,  we have, 
$$ x^* - x^t \leq \frac{2c\delta}{(1+\delta)(1-c)}.$$  
\end{restatable}
\input{AAMAS25/Proof_Lemma7}

\begin{manualtheorem}{2}[Part (1): Social Welfare Regret of \textsc{Explore-First}] For any feasible MA-MAB instance $\mathcal{I}$ with $T$ sufficiently large and $C_i=c \forall i\in [n]$,     expected social welfare regret of \textsc{Explore-First} algorithm is upper-bounded by $O\left(\frac{n}{a_{\min} } T^{2/3} \sqrt{\log(T)}\right)$. Here $a_{min}:= \min_{i \in[n],j \in [m]} A_{i,j}$.  
\end{manualtheorem}
\begin{proof}
The social welfare regret for 2-arm MA-MAB can be written as 
$$\mathcal{R}_\textsc{SW}(T) =  \sum_{t=1}^T (x^* -x^t) \Delta. 
$$

Here, $\Delta: = \sum_{i=1}^n A_{i,1} - \sum_{i=1}^n A_{i,2} $ ($0<\Delta \leq n$) denote the difference between the social welfare generated by optimal  arm 1 and suboptimal arm 2. The expectation is taken over the instance, and also randomness in the policy.

Let $\widehat A$ denote the sample estimate of $A$ after the exploration phase i.e. first $ 2t' :=\lfloor T^\alpha \rfloor$ rounds. \footnote{We assume that $\lfloor T^\alpha \rfloor$ is even number. However, otherwise we can consider $\lfloor T^\alpha \rfloor$ +1 with additional regret of 1. This does not affect our regret guarantee in order terms.     }  
Furthermore, let the arm pull distribution of \textsc{Explore-First} algorithm  be denoted by $[x^t, 1-x^t]$ and $[x^*, 1-x^*]$ be  optimal feasible policy. 



Recall that the \textsc{Explore-First} algorithm pulls arm 1 with probability $1/2$ till first $2t'$ round and with probability $x^t$ for all $t\geq 2t'$. We now bound the expected social welfare regret.
\begin{align*}
\mathbb{E} \left[\mathcal{R}_\textsc{SW}(T)\right] & =  \sum_{t=1}^T \mathbb{E}[(x^* -x^t)  \Delta]\\ 
& = \sum_{t=1}^{2t'} (x^* - 1/2) \Delta + \sum_{t=2t' + 1}^{T} (\mathbb{E}[x^* -x^t]) \Delta \\
& \leq 2nt' + \left( \sum_{t=2t' + 1}^{T} \frac{2c\delta}{(1+\delta)(1-c)}  \mathbb{P}\Big(x^* - x^t \le \frac{2c\delta}{(1+\delta)(1-c)} \Big)  +  \sum_{t=2t' + 1}^{T}  \mathbb{P}\Big(x^* - x^t >  \frac{2c\delta}{(1+\delta)(1-c)} \Big)      \right) \Delta \tag{From Lemma~\ref{lem:lemOne}}\\  
& \leq 2nt' + n (T - 2t') \Big[ \left(1 - 2e^{-t'\delta^2a_{\min}^2/2\sigma^2}\right) \frac{2\delta c}{(1+\delta)(1-c)} +  2e^{-t' \delta^2  a_{\min}^2/2\sigma^2} \Big] \tag{From Lemma~\ref{lem:lemOne}}
\intertext{Choose $2t' = T^{2/3}$ and $\delta = \left(\frac{2\sigma}{ a_{\textup{min}}}\right)\sqrt{\frac{\log{T}}{3T^{2/3}}}$ with $T$ sufficiently large for $\delta\in (0, 1)$ to get,}  
\mathbb{E}\left[\mathcal{R}_\textsc{SW}(T)\right] & \leq n \left( T^{2/3} +  \frac{4cT^{2/3}\sqrt{\log{T}}}{(1-c)(a_{\min}/\sigma)} + T^{2/3} \right) \\ 
&\leq n \left( T^{2/3} +  \frac{4T^{2/3}\sqrt{\log{T}}}{(a_{\min}/\sigma)} + T^{2/3} \right) \tag{with the sufficient condition for feasibility, $c\in (0, 1/2]$}\\
&\leq 6\frac{n\sigma}{a_{\min}}T^{2/3}\sqrt{\log{(T)}}
\end{align*}
This proves the claim. Our proof needs $c<1$, which is justified because $c=1$ would simplify the corresponding fairness constraint to be $ \langle A_i,\pi\rangle = \Amax_i$, which in turn fixes certain coordinates of $\pi$ making the problem simpler.
\end{proof}

Next, we bound the fairness regret of \textsc{Explore-First} algorithm. We begin with a crucial lemma.
\begin{lemma}\label{fairness-reform}
The Fairness regret for the case with two arms simplifies as \begin{equation}
\mathcal{R}_\textsc{FR}(T) \leq \sum_{t=1}^T\sum_{i=1}^n \left|(A_{i,1}-A_{i,2})\cdot (x^*-x^t)\right|_+.
\end{equation}
\end{lemma}
\begin{proof}
This proof uses the characterization of the feasible policy presented in Proposition (\ref{prop:One}), restated below.
\begin{equation}
    \max \Bigg( 0,  \max_{i \in [n_1]}\frac{C_i - \frac{A_{i,2}}{A_{i,1}}}{ 1 - \frac{A_{i,2}}{A_{i,1}}} \Bigg) \leq x \leq \min \Bigg( 1, \min_{i \in [n] \setminus [n_1]} \frac{1 - C_i}{ 1 - \frac{A_{i,1}}{A_{i,2}}} \Bigg).
    \label{feasible-x}
\end{equation}

We first simplify $C_i\Amax_i-A_i^\top \pi^t$. Let $e^i\in \{0, 1\}^m$  be such that $e^i_j=1$, for $j\in\argmax_{j\in[m]}\{A_{i,j}\}$ having the least index.

We have, 
\begin{align}
    C_i\Amax_i-\langle A_i, \pi^t\rangle &=A_{i,1}e^i_1C_i+A_{i,2}(1-e^i_1)C_i - A_{i,1}x^t - A_{i,2}(1-x^t) \nonumber \\
    &=A_{i,1}\left( C_ie^i_1-x^t \right)-A_{i,2}\left(C_ie^i_1-x^t+1-C_i\right) \nonumber \\
    &= \left(A_{i,1}-A_{i,2}\right).\left( C_ie^i_1-x^t \right)-A_{i,2}(1-C_i)\nonumber \\
    & = (A_{i,1}-A_{i,2})\left( (C_ie^i_1-x^t)-\frac{A_{i,2}}{A_{i,1}-A_{i,2}}(1-C_i)\right)\label{simplify}
\end{align}
An implicit assumption is that $A_{i, 1}\neq A_{i,2}\forall i\in [n]$. This is without loss of generality because when $A_{i, 1}= A_{i,2}$, agent $i$ does not contribute to the fairness regret.
We now separately analyze the case for $i\in[n_1]$ and $i\in[n]\setminus [n_1]$.
\newline
\textbf{Case}: For $i\in[n_1]$ i.e. for $\{i\in [n]\ | A_{i,1}>A_{i,2}\}$. For such $i$'s, $e^i_1=1$, with which (\ref{simplify}) simplifies to the following.
\begin{align}
C_i\Amax_i-\langle A_i, \pi^t\rangle & = (A_{i,1}-A_{i,2})\left( (C_i-x^t)-\frac{A_{i,2}}{A_{i,1}-A_{i,2}}(1-C_i)\right)\nonumber \\
&=(A_{i,1}-A_{i,2})\left( \frac{C_iA_{i,1}-C_iA_{i,2}-A_{i,2}+C_iA_{i,2}}{A_{i,1}-A_{i,2}}-x^t \right)\nonumber \\
&=(A_{i,1}-A_{i,2})\left( \frac{C_iA_{i,1}-A_{i,2}}{A_{i,1}-A_{i,2}}-x^t \right)\nonumber \\
&=(A_{i,1}-A_{i,2})\left( \frac{C_i-\frac{A_{i,2}}{A_{i,1}}}{1-\frac{A_{i,2}}{A_{i,1}}}-x^t \right)\nonumber \\
& \leq (A_{i,1}-A_{i,2})\cdot (x^*-x^t) \quad \tag{From (\ref{feasible-x})}. \label{FR-i1}
\end{align}
\newline
\textbf{Case}: $i\in[n]\setminus [n_1]$ i.e. for $\{i\in [n]\ | A_{i,1}<A_{i,2}\}$. For such $i$'s, $e^i_1=0$, with which (\ref{simplify}) simplifies to the following.
\begin{align}\label{FR-i2}
C_i\Amax_i-\langle A_i, \pi^t\rangle & = (A_{i,1}-A_{i,2})\left( -x^t-\frac{A_{i,2}}{A_{i,1}-A_{i,2}}(1-C_i)\right)\nonumber \\
& = (A_{i,1}-A_{i,2})\left( -x^t + \frac{1-C_i}{1-\frac{A_{i,1}}{A_{i,2}}} \right)\nonumber \\
&\leq (A_{i,1}-A_{i,2})\cdot (-x^t+x^*) \quad \tag{From (\ref{feasible-x}), and using that $(A_{i,1}-A_{i,2})<0$} 
\end{align}
Combining the above two cases, we have that $C_i\Amax_i-\langle A_i, \pi^t\rangle\leq (A_{i,1}-A_{i,2})(x^*-x^t)$ which proves $\mathcal{R}_{\textsc{FR}}(T)=\sum_{t=1}^T\sum_{i=1}^n|C_i\Amax_i-\langle A_i, \pi^t\rangle|_+\leq \sum_{t=1}^T\sum_{i=1}^n(A_{i,1}-A_{i,2})(x^*-x^t)$.    
\end{proof}

\input{AAMAS25/fairness-regret_ef}

\section{Fairness and Social Welfare Regret Guarantees of \textsc{RewardFairUCB}}

We first bound the social welfare regret of \ouralgo . We begin with an important lemma from \cite{wang2021fairness}.   
  \begin{restatable}{lemma}{MartingaleLemma}[Lemma A.4.2 in \cite{wang2021fairness}]\label{lemma:wang21}
For any $\delta \in (0,1)$, with probability $1 - \delta/2$ we have 
\begin{equation}
    \Big |\sum_{t=\tau+1}^T \mathbb{E}_{j \sim \pi^t} \sqrt{1/N_j^t} - \sum_{t=\tau+1}^T \sqrt{1/N_{j_t}^{t}} \Big | \leq \sqrt{2T\log(4/\delta)}
\end{equation}
where $N_j^t$ represents the number of times arm $j$ has been selected up to time $t$, and $j_t$ denotes the arm actually chosen at time $t$.
\end{restatable}
\begin{proof} We re-write the proof of [Lemma A.4.2 in \cite{wang2021fairness}].
Consider the sequence $\{Z^t\}_{t=\tau+1}^T$ with 
$
    Z^t := \sqrt{\frac{1}{N_{j_t}^t}} - \mathbb{E}_{j \sim \pi^t} \sqrt{\frac{1}{N_j^t}}
$. This sequence is a martingale difference sequence. Additionally, for any $t > \tau$, we have
\[
    |Z^t| = \left|\sqrt{\frac{1}{N_{j_t}^t}} - \mathbb{E}_{j \sim \pi^t} \sqrt{\frac{1}{N_j^t}}\right| \leq 1,
\]
since the square root function is Lipschitz with a constant of 1, and the counts $N_j^t$ are non-negative.

Thus, we can apply Azuma-Hoeffding's inequality to bound the sum of the martingale difference sequence:
\[
    \mathbb{P} \left( \Bigg| \sum_{t=\tau+1}^T Z^t \Bigg| \geq \epsilon \right) \leq 2 \exp\left( -\frac{\epsilon^2}{2(T-\tau)} \right).
\]
Setting $\epsilon = \sqrt{2(T-\tau) \log(4/\delta)}$ gives that with probability at least $1 - \delta/2$,
\[
    \Bigg| \sum_{t=\tau+1}^T Z^t \Bigg| \leq \sqrt{2T \log\left(\frac{4}{\delta}\right)}.
\]
This concludes the proof.
\gan{@Himanshu: Write the proof of the lemma}
\end{proof}
\rewardRegretUCB*

 We begin the proof of the Theorem with an important lemma. 
\begin{lemma}
    Let $T$ be an arbitrary stopping time and  $(j^t)_{t\leq T}$ denote the sequence of arms pulled by the algorithm till time $T$. Furthermore, let $N_{j}^t$ denotes the number of time instances a fixed arm $j \in [m]$ was  pulled till time instance $t$.   Then 
    \begin{align}
        \sum_{t=1}^T \frac{1}{\sqrt{N_{j^t}^t}} \leq 2 \sqrt{mT}. 
    \end{align}
    \label{lem:Supp_Lemma_nine}
\end{lemma}
\begin{proof}[Proof of the Lemma]
\begin{align*}
   \sum_{t=1}^T \frac{1}{\sqrt{N_{j^t}^t}} & = \sum_{t=1}^T \sum_{j=1}^m \frac{1}{\sqrt{N_{j}^t}} \mathds{1}(j^t = j) = \sum_{j=1}^m \sum_{t=1}^T   \frac{1}{\sqrt{N_{j}^t}} \mathds{1}(j^t = j) \\ \intertext{Let $T_j$ be the total number of times arm $j$ was pulled till  time horizon  $T$, then }
   \sum_{j=1}^m \sum_{t=1}^T   \frac{1}{\sqrt{N_{j}^t}} \mathds{1}(j^t = j) & = \sum_{j=1}^m \sum_{\ell=1}^{T_j}   \frac{1}{\sqrt{\ell}}  =  \sum_{j=1}^m \sum_{\ell=1}^{T_j} \int_{x = \ell-1}^\ell \frac{1 }{\sqrt{\ell}} \textup{d}x \leq \sum_{j=1}^m \sum_{\ell=1}^{T_j} \int_{x = \ell-1}^\ell \frac{1 }{\sqrt{x} } \textup{d}x  \leq \sum_{j=1}^m  2 \sqrt{T_j} \leq 2 \sqrt{mT}
\end{align*}
The last inequality follows from Cauchy-Schwarz inequality and the fact that $T = \sum_{j=1}^m T_j$. 
\end{proof}
 \begin{proof}[Proof of the Theorem]
\begin{align}
    \mathcal{R}_\textsc{SW}(T) &= \sum_{t=1}^T \big[ SW_{\pi^*}(A)  - SW_{\pi^t}(A)  \big] \nonumber \\ 
&\leq  \underbrace{\sum_{t=1}^{ m\lceil \sqrt{T} \rceil} \sum_{i=1}^n \langle A_i, \pi^* - \pi^t \rangle}_{\textup{R1}}   + \underbrace{ \sum_{t= m\lceil \sqrt{T } \rceil + 1 }^T \big[ SW_{\pi^*}(A)  - SW_{\pi^t}(A)  \big] }_{\textup{R2}} 
\end{align}
We bound terms R1 and R2 separately. We begin with an upper bound for R1. 
\begin{align*}
 \textup{R1} &=  \sum_{t=1}^{ m\lceil \sqrt{T} \rceil} \sum_{i=1}^n \langle A_i, \pi^* - \pi^t \rangle 
 \\ &=  \sum_{j =1}^m\sum_{t \leq (m\lceil \sqrt{T} \rceil \land j^t =j) } \  \sum_{i=1}^n   A_{i,j} ( \pi_j^* - \pi^t_j )  \tag{where $j^t$ denotes the arm chosen at timestep $t$}\\
&\leq \sum_{j =1}^m\sum_{t \leq (m\lceil \sqrt{T} \rceil \land j^t =j) }\  \sum_{i=1}^n   A_{i,j}  \pi_j^*    \\ 
&\leq  \sqrt{T}  \sum_{j =1}^m \Big (\sum_{i=1}^n   A_{i,j} \Big )  \pi_j^* 
\\ & = \sqrt{T} \left\langle \sum_{i=1}^n   A_i, \pi^*  \right\rangle   \\ 
&\leq  \sqrt{T}  \left\| \sum_{i=1}^n   A_{i}\right\|_2  \|\pi^*\|_2    \\ 
&\leq \sqrt{T} n\sqrt{m}\sqrt{m}  = mn\sqrt{T}.
\end{align*}
Next, we give an upper bound on R2. The proof uses the notation $\epsilon^t_{j}=\frac{\epsilon}{\sqrt{N_j^t}}=\sigma \sqrt{\frac{2\log{(8mn/\delta)}}{N_j^t}}$. 
\begin{align*}
   \textup{R2} &= \sum_{t= m\lceil \sqrt{T } \rceil + 1 }^T \big[ SW_{\pi^*}(A)  - SW_{\pi^t}(A)  \big] \\ 
   & \leq \sum_{t= m\lceil \sqrt{T } \rceil + 1 }^T \big[ SW_{\pi^t}(\overline{A}^t)  - SW_{\pi^t}(A)  \big] \ \tag{w.p. at least $1-\delta/2$; From Lemma \ref{lem:equivalence}}\\  
  & \leq  \sum_{t= m\lceil \sqrt{T } \rceil +1}^T \sum_{i=1}^n \langle \overline{A}_i^t - A_i,  \pi^t \rangle \\ 
    & \leq 2n \sum_{t=m+1}^T \mathbb{E}_{j \sim \pi^t} \left[\varepsilon^{t}_{j}\right]   \\
    &=   2 n \varepsilon \sum_{t=m+ 1}^T \mathbb{E}_{j \sim \pi^t}\left[\sqrt{1/N_j^t}\right] \\ 
& \leq   2n \varepsilon \left( \sum_{t=1}^T \sqrt{1/N_{j^t}^{t}} + \sqrt{2T\log(2/\delta)} \right) \ \tag{w.p. at least $1-\delta/2$; From Lemma \ref{lemma:wang21}, where $j^t$ denotes the arm chosen at $t$}\\
& \leq  2n\sigma \sqrt{2\log(8mn/\delta)}  ( 2 \sqrt{mT} + \sqrt{2T\log(2/\delta)}) \tag{from Lemma \ref{lem:Supp_Lemma_nine} above} \\
&\leq 2n\sigma \left(2 \sqrt{2mT\log{(8mn/\delta)}} + \sqrt{2T}\log{(8mn/\delta)} \right)\\
&\leq 6n\sigma  \sqrt{2mT}\log{(8mn/\delta)}
\end{align*}
We have the following bound on the expected regret
\begin{align*}
(1-\delta)(R_1+R_2)+\delta T \leq (1-\delta)(6n\sigma \sqrt{2mT}\log{(8mn/\delta)} + mn\sqrt{T})  + \delta T\leq 6n\sigma \sqrt{2mT}\log{(8mn/\delta)} + mn\sqrt{T} + \delta T.
\end{align*}
We choose $\delta=\frac{6n\sigma \sqrt{2m} }{\sqrt{T}}$ (with $T$ large enough to have $\delta<1$) to get the tightest upper bound on the RHS. This proves the stated upper-bound on the social welfare regret of \ouralgo.

\end{proof}

\input{AAMAS25/ucb_fr}
\section{Lower Bounds}
\lowerRegret*
\begin{proof}
In addition to the instance described in Sec~\ref{subsec:lowerBound} for proving the lower bound on the social welfare, we observe that our MA-MAB problem for a single agent ($n=1$) and $C$ as the zero matrix reduces to the Nash Social Welfare based MA-MAB problem \cite{Hossain2020FairAF} with $n=1$. The lower bound of $\Omega(\sqrt{mT})$ for social welfare regret then follows from Proposition (2) in Appendix A of \citep{Hossain2020FairAF}.

    We now discuss the lower bound for the fairness regret. We observe that it is enough to show the lower bound for $m=2$ case as we can easily construct instances of the problem with more than 2 arms that reduce to the case with $m=2$, e.g. by choosing an $A$ matrix with $m-2$ columns as zeros. The proof of lower bound for the fairness regret in the 2-arm case proceeds as follows. From Lemma~\ref{fairness-reform}, we have that $\mathcal{R}_{FR}(T) \leq \sum_{t=1}^T\sum_{i=1}^n \left|(A_{i,1}-A_{i,2}).(x^*-x^t)\right|_+$. To prove the lower bound, we construct an instance where this inequality is tight. We first recall an inequality from the proof of characterization of the optimal feasible policy ($[x, 1-x]$) in Lemma~\ref{prop:One}.
    \begin{align}\label{charac-for-lb}
    \max \Bigg( 0,  \max_{i \in [n_1]}\frac{C_i - \frac{A_{i,2}}{A_{i,1}}}{ 1 - \frac{A_{i,2}}{A_{i,1}}} \Bigg) \leq x \leq \min \Bigg( 1,  \min_{i \in [n] \setminus [n_1]}\frac{1 - C_i}{ 1 - \frac{A_{i,1}}{A_{i,2}}} \Bigg). 
    \end{align}
    Also, following the steps in the proof of Lemma~\ref{fairness-reform}, we have the following.
    \begin{align}\label{FR-lb}
        \mathcal{R}_{FR}(T) &= \sum_{t=1}^T\left( \sum_{i\in [n_1]}|C_i\Amax_i - \langle A_i, \pi^t \rangle|_+ + \sum_{i\in [n]\setminus [n_1]}|C_i\Amax_i - \langle A_i, \pi^t \rangle|_+\right) \nonumber \\
        &= \sum_{t=1}^T\left(\sum_{i\in [n_1]}
        \left| (A_{i, 1}-A_{i, 2})\left(\frac{C_i-\frac{A_{i, 2}}{A_{i, 1}}}{1-\frac{A_{i, 2}}{A_{i, 1}}}-x^t \right)\right|_+
        +
        \sum_{i\in [n]\setminus [n_1]} \left|(A_{i, 1}-A_{i, 2})\left( \frac{1-C_i}{1-\frac{A_{i, 1}}{A_{i, 2}}} - x^t \right)\right|_+ \right)
    \end{align}
    Consider an instance where $\frac{C_i-\frac{A_{i, 2}}{A_{i, 1}}}{1-\frac{A_{i, 2}}{A_{i, 1}}} = l(>0) \ \forall i\in [n_1]$ and $\frac{1-C_i}{1-\frac{A_{i, 1}}{A_{i, 2}}}=u(<1) \ \forall i\in [n]\setminus [n_1]$ and $l=u$. In this case, $x^*=l=u$ and the fairness regret (Eq.~\ref{FR-lb}) reduces to the following.
    \begin{align*}
        \mathcal{R}_\textsc{FR}(T) &= \sum_{t=1}^T\sum_{i=1}^n \left|(A_{i,1}-A_{i,2})\cdot (x^*-x^t)\right|_+ \\
        &\geq   \left|\sum_{t=1}^T\sum_{i=1}^n (A_{i,1}-A_{i,2})\cdot (x^*-x^t) \right|_+ \ \tag{Triangular inequality}\\
        &= \left| \mathcal{R}_\textsc{SW}(T) \right|_+ \tag{From Eq.~\ref{eq:eqnClosedFormSW}}\\
        &\geq   \Omega(\sqrt{T})  \ \tag{From the lower bound for social welfare}.
    \end{align*}
    Finally, we present an instance where the fairness regret inequality is tight. Let $\underline{i}\in \argmax_{i\in [n_1]}\left(\frac{C_i-\frac{A_{i, 2}}{A_{i, 1}}}{1-\frac{A_{i, 2}}{{A_{i, 1}}}} \right)$ and $\overline{i}\in \left(\argmin_{i\in [n]\setminus [n_1]} \frac{1 - C_i}{ 1 - \frac{A_{i,1}}{A_{i,2}}}\right)$ be the least indices. Consider an instance with $C_{\underline{i}} = 1-C_{\overline{i}}$ and an $A$ matrix such that $A_{\underline{i}, 2}=0$ and $A_{\overline{i}, 1}=0$, e.g. $A=\begin{bmatrix}
    1 & 0\\
    0 & 1\\
    1 & 0
    \end{bmatrix}$. For this instance, any optimal feasible policy $[x^*, 1-x^*]$ must satisfy $C_{\underline{i}}\leq x^* \leq 1-C_{\overline{i}} \implies x^* = C_{\underline{i}}$.

\end{proof}
\input{AAMAS25/simulation_Appendix}

%% file: AAMAS25/Proof_Lemma7.tex
\begin{proof}[Proof of lemma~\ref{lem:lemOne}]
Let $i^*\in \argmin_{i \in [n] \setminus [n_1]}\left(\frac{1 - c}{ 1 - \frac{{A}_{i,1}}{{A}_{i,2}}}\right)$ and $i'\in \argmin_{i :\widehat{A}_{i, 1}<\widehat{A}_{i, 2}}\left(\frac{1 - c}{ 1 - \frac{\widehat{A}_{i,1}}{\widehat{A}_{i,2}}}\right)$.
\begin{align*}
    x^*-x^t&= \min \left( 1,  \min_{i \in [n] \setminus [n_1]}\frac{1 - c}{ 1 - \frac{A_{i,1}}{A_{i,2}}} \right)-\min \left( 1, \min_{i:\widehat{A}_{i, 1}<\widehat{A}_{i, 2}} \frac{1 - c}{ 1 - \frac{\widehat{A}_{i,1}}{\widehat{A}_{i,2}}} \right)
\end{align*}
When $\{i: \widehat{A}_{i, 1}<\widehat{A}_{i, 2}\}\neq \phi$, $x^*-x^t = \min\left(1, \frac{1-c}{1-\frac{A_{i^*,1}}{{A_{i^*,2}}}}\right)
    - \min\left(1, \frac{1-c}{1-\frac{\widehat{A}_{i',1}}{{\widehat{A}_{i',2}}}}\right),$ else, $x^*-x^t=\min\left(1, \frac{1-c}{1-\frac{A_{i^*,1}}{{A_{i^*,2}}}}\right)
    - 1.$
\newline
We bound $x^*-x^t$ under the event $E_\delta=\left\{\frac{\widehat{A}_{i', 1}}{\widehat{A}_{i', 2}}\leq \frac{A_{i', 1}}{A_{i', 2}}-\frac{2c\delta}{1+\delta}\right\}$ for $\delta\in \left(0, 1\right)$ and $i'\in [n]$ when $\frac{\widehat{A}_{i', 1}}{\widehat{A}_{i', 2}}<c$. As $\frac{\widehat{A}_{i', 1}}{\widehat{A}_{i', 2}}<c$, the event $E_\delta$ implies $\frac{\widehat{A}_{i', 1}}{\widehat{A}_{i', 2}}\leq \frac{1+\delta}{(1+3\delta)}\frac{A_{i', 1}}{A_{i', 2}}$.
Further, $E_\delta \cap \left\{ \widehat{A}_{i', 2} \leq  \left(1+3\delta \right)A_{i', 2} \right\}\subseteq \left\{ \widehat{A}_{i', 2} \leq  \left(1+3\delta \right)A_{i', 2} \right\}\cap \left\{ \widehat{A}_{i', 1} \leq  \left(1+\delta \right)A_{i', 1} \right\}.$
\begin{align*}
    \mathbb{P}\left(E_\delta \cap \left\{ \widehat{A}_{i', 2} \leq  \left(1+3\delta \right)A_{i', 2} \right\}\right) 
    &\leq \mathbb{P}\left(\left\{ \widehat{A}_{i', 2} \leq  \left(1+3\delta \right)A_{i', 2} \right\}\right) \mathbb{P}\left(\left\{ \widehat{A}_{i', 1} \leq  \left(1+\delta \right)A_{i', 1} \right\}\right) \ \tag{From independence of $\widehat{A}_{i', 1} \ \& \ \widehat{A}_{i', 2}$} \\
    &\leq \mathbb{P}\left(\left\{ \widehat{A}_{i', 1} \leq  \left(1+\delta \right)A_{i', 1} \right\}\right)  \\
    &\leq \exp\left( -\frac{t'}{2}\left(A_{i', 1}/\sigma\right)^2\delta^2 \right) \ \tag{From Lemma~\ref{lemma:chernoff-utb}}
\end{align*}
Also, $\mathbb{P}\left(E_\delta \cap \left\{ \widehat{A}_{i', 2} >  \left(1+3\delta \right)A_{i', 2} \right\}\right) \leq  \mathbb{P}(\widehat{A}_{i', 1}\leq (1+\delta) A_{i', 1}) \leq \exp\left( -\frac{t'}{2}\left(A_{i', 1}/\sigma\right)^2\delta^2 \right)$ (Lemma~\ref{lemma:chernoff-utb}). Thus, we have the following.
\begin{align*}
    \mathbb{P}(E_\delta)&=  \mathbb{P}\left(E_\delta \cap \left\{ \widehat{A}_{i', 2} \leq  \left(1+3\delta \right)A_{i', 2} \right\}\right) + \mathbb{P}\left(E_\delta \cap \left\{ \widehat{A}_{i', 2} >  \left(1+3\delta \right)A_{i', 2} \right\}\right)\\
    &\leq  2\exp\left( -\frac{t'}{2}\left(A_{i', 1}/\sigma\right)^2\delta^2 \right)\\
    &\leq 2\exp\left( -\frac{t'}{2}\left(a_{\min}/\sigma\right)^2\delta^2 \right).
\end{align*}
Now, we begin the proof of upper-bounding $x^*-x^t$.
\newline
\textbf{Case (I) $x^*<1$,\ $x^t<1$:}
We first note that $x^*<1\implies \frac{A_{i^*, 1}}{A_{i^*, 2}}<c$ and $x^t<1\implies \frac{\widehat{A}_{i', 1}}{\widehat{A}_{i', 2}}<c$.
\begin{align*}
    x^*-x^t & = \frac{1-c}{1-\frac{A_{i^*,1}}{A_{i^*,2}}}-\frac{1-c}{1-\frac{\widehat{A}_{i',1}}{\widehat{A}_{i',2}}} \\
    &\leq \frac{1-c}{1-\frac{A_{i',1}}{A_{i',2}}}-\frac{1-c}{1-\frac{\widehat{A}_{i',1}}{\widehat{A}_{i',2}}} \ \tag{From the definition of $i^*$ and $\widebar{E_\delta}$}
\end{align*}
The last inequality follows by noting that under $\widebar{E_\delta},\ i'\in ([n]\setminus [n_1])$ and then we use the fact that $i^*$ is the argmin obtained after searching over the set $([n]\setminus [n_1])$. Thus, with probability at least $1-\mathbb{P}(E_\delta)$, we have the following
\begin{align*}
    x^*-x^t&\leq (1-c)\frac{\frac{2c\delta}{1+\delta}}{(1-c)\left(1-\frac{A_{i', 1}}{A_{i', 2}} \right)}\\
    &\leq (1-c)\frac{\frac{2c\delta}{1+\delta}}{(1-c)\left(1-\frac{\widehat{A}_{i', 1}}{\widehat{A}_{i', 2}} +\frac{2c\delta}{1+\delta} \right)} \ \tag{under $\widebar{E_\delta}$}\\
    &\leq \frac{2c\delta}{1+\delta}.
\end{align*}
\textbf{Case (II) $x^*<1$,\ $x^t=1$:} Here, $x^*-x^t<0$.
\newline
\textbf{Case (III) $x^*=1$,\ $x^t<1$}: 
We first note that $x^*=1\implies \frac{A_{i^*, 1}}{A_{i^*, 2}}\geq c$ as $[n]\setminus [n_1]$ is always taken to be non-empty. 

\begin{align*}
x^* - x^t = \frac{c - \frac{\widehat{A}_{i',1}}{\widehat{A}_{i',2}}}{1 - \frac{\widehat{A}_{i',1}}{\widehat{A}_{i',2}}}
\end{align*}
As RHS is a decreasing function in $\frac{\widehat{A}_{i', 1}}{\widehat{A}_{i', 2}}$, with probability $1-\mathbb{P}(E_\delta)$, we have the following.
\begin{align*}
    x^*-x^t\leq & \frac{c-\frac{A_{i', 1}}{A_{i', 2}}+\frac{2c\delta}{1+\delta}}{1-\frac{A_{i', 1}}{A_{i', 2}}+\frac{2c\delta}{1+\delta}}
\end{align*}
This RHS is a decreasing function of $\frac{A_{i', 1}}{A_{i', 2}}$. We know that if $i'\in [n]\setminus [n_1]$, $\frac{A_{i', 1}}{A_{i', 2}}\geq \frac{A_{i^*, 1}}{A_{i^*, 2}}\geq c$. Also, if $i'\in [n_1]$, $\frac{A_{i', 1}}{A_{i', 2}}\geq 1\geq c$. Thus, we have the following,
\begin{align*}
    x^*-x^t\leq & \frac{\frac{2c\delta}{1+\delta}}{1-c+\frac{2c\delta}{1+\delta}}=\frac{2c\delta}{(1-c)+\delta(1+c)} \leq \frac{2c\delta}{(1-c)(1+\delta)}.
\end{align*}
This proves our claim that with probability $\mathbb{P}(\bar{E_\delta})$ i.e. with probability at least $\left(1-2\exp\left( -\frac{t'}{2}\left(a_{\min}/\sigma\right)^2\delta^2 \right)\right)$, we have $x^*-x^t\leq \frac{2c\delta}{(1-c)(1+\delta)}.$

\end{proof}

%% file: AAMAS25/fairness-regret_ef.tex
\begin{manualtheorem}{2}[Part (2): Fairness Regret of \textsc{Explore-First}]
\label{thm:ExploreFirstFairnessRegret}
For any feasible MA-MAB instance $\mathcal{I}$ with $T$ sufficiently large and $C_i=c \forall i\in [n]$,     expected fairness regret of \textsc{Explore-First} algorithm is upper-bounded by $O\left(\frac{n}{ a_{\min}} T^{2/3} \sqrt{\log(T)} \right)  $.  Here $a_{\min}:=\min_{i\in [n], j\in [m]}A_{i, j}$.
\end{manualtheorem}

\begin{proof}
Using Lemma (\ref{fairness-reform}) and the fact that $\pi^t = [1/2, 1/2]$ for $t<2t'$, the fairness regret of \textsc{Explore-First} is given as follows.
\begin{align} \label{proof:FR-EF}
\mathcal{R}_{FR}(T) & \leq \sum_{t=1}^T\sum_{i=1}^n \left|(A_{i,1}-A_{i,2}).(x^*-x^t)\right|_+ \nonumber \\
    &  = \sum_{i=1}^n\sum_{t=1}^{2t'}|c\Amax_i - A_{i, 1}/2-A_{i, 2}/2|_++\sum_{t=2t'+1}^T\sum_{i=1}^n\max\Big\{ \left|A_{i,1}-A_{i,2}\right|_+\left|x^*-x^t\right|_+,\ \left|A_{i,2}-A_{i,1}\right|_+  \left|x^t - x^*\right|_+\Big\} \nonumber \\
    &  = \sum_{t=2t'+1}^T\sum_{i=1}^n\max\Big\{ \left|A_{i,1}-A_{i,2}\right|_+\left|x^*-x^t\right|_+,\ \left|A_{i,2}-A_{i,1}\right|_+  \left|x^t - x^*\right|_+\Big\} \nonumber \tag{with the sufficient condition for feasibility, $c\in (0, 1/2]$}\\
    & \leq \sum_{t=2t'+1}^T (n-n_1) |x^t - x^*|_+ +  n_1  |x^* - x^t|_+ .
\end{align}
Similar to the bound on $x^*-x^t$ derived in Lemma~\ref{lem:lemOne}, we bound $x^t-x^*$ taking different cases.
\newline
We first upper-bound the probability of the event $F_\delta=\left\{ \frac{\widehat{A}_{i^*, 1}}{\widehat{A}_{i^*, 2}} - \frac{2c\delta}{1+\delta} \geq \frac{A_{i^*, 1}}{A_{i^*, 2}} \right\}$ for $\delta\in (0, 1)$ and $i^*\in [n]$ when $\frac{A_{i^*, 1}}{A_{i^*, 1}}<c$, which will be used in our proof. As $\frac{A_{i^*, 1}}{A_{i^*, 2}}<c$, $F_\delta$ implies $\frac{\widehat{A}_{i^*, 1}}{\widehat{A}_{i^*, 2}}\geq \frac{A_{i^*, 1}}{A_{i^*, 2}}\left(\frac{1+3\delta}{1+\delta} \right)$. 
\newline
Further, $F_\delta \cap \{\widehat{A}_{i^*, 1}\leq A_{i^*, 1}(1+3\delta) \} \subseteq \{\widehat{A}_{i^*, 1}\leq A_{i^*, 1}(1+3\delta) \}\cap \{ \widehat{A}_{i^*, 2}\leq (1+\delta)A_{i^*, 2} \}$.
\begin{align*}
    \mathbb{P}\left(F_\delta \cap \{\widehat{A}_{i^*, 1}\leq A_{i^*, 1}(1+3\delta) \}\right) &\leq \mathbb{P}\left(\{\widehat{A}_{i^*, 1}\leq A_{i^*, 1}(1+3\delta) \}\right) \mathbb{P}\left(\{ \widehat{A}_{i^*, 2}\leq (1+\delta)A_{i^*, 2} \}\right) \tag{From independence of $\widehat{A}_{i^*, 1}$ \& $\widehat{A}_{i^*, 2}$}\\
    &\leq \mathbb{P}\left(\{ \widehat{A}_{i^*, 2}\leq (1+\delta)A_{i^*, 2} \}\right)\\
    &\leq \exp{\left(-\frac{t'}{2}(A_{i^*, 2}/\sigma)^2\delta^2 \right)} \tag{From Lemma~\ref{lemma:chernoff-utb}}
\end{align*}
Also, $\mathbb{P}(F_\delta \cap \{\widehat{A}_{i^*, 1}> A_{i^*, 1}(1+3\delta) \})\leq \mathbb{P}(\{\widehat{A}_{i^*, 2}\leq (1+\delta)A_{i^*, 2}\})\leq \exp{\left(-\frac{t'}{2}(A_{i^*, 2}/\sigma)^2\delta^2 \right)}$ (Lemma~\ref{lemma:chernoff-utb}). Thus, we have the following.
\begin{align*}
    \mathbb{P}(F_\delta) &= \mathbb{P}(F_\delta\cap \{\widehat{A}_{i^*, 1}\geq A_{i^*, 1}(1+3\delta) \}) + \mathbb{P}(F_\delta\cap \{\widehat{A}_{i^*, 1}< A_{i^*, 1}(1+3\delta) \})\\
    & \leq 2\exp\left(-\frac{t'}{2}(A_{i^*, 2}/\sigma)^2\delta^2 \right)
    \\
    &\leq  2\exp\left(-\frac{t'}{2}(a_{\min}/\sigma)^2\delta^2 \right).
\end{align*}
From Lemma~\ref{lem:lemOne}, we also recall that $E_\delta=\{\frac{\widehat{A}_{i', 1}}{\widehat{A}_{i', 2}}\leq \frac{A_{i', 1}}{A_{i', 2}}-\frac{2c\delta}{1+\delta}\}$ for $\delta \in (0, 1)$ and $i'\in [n]$.
We now consider different cases for bounding $x^t-x^*$ and subsequently the fairness regret.
\newline
\textbf{Case (I) $x^*<1, x^t<1$:} 
From Lemma~\ref{lem:lemOne}, we have that $|x^*-x^t|_+\leq \frac{2c\delta}{(1-c)(1+\delta)}$ under $\widebar{E_\delta}$.
For bounding, $x^t-x^*$, we first note that, $x^*<1\implies \frac{A_{i^*, 1}}{A_{i^*, 2}}<c$ and $x^t<1\implies \frac{\widehat{A}_{i', 1}}{\widehat{A}_{i', 2}}<c$.
\begin{align*}
    x^t-x^* = \frac{1-c}{1-\frac{\widehat{A}_{i',1}}{\widehat{A}_{i',2}}}-\frac{1-c}{1-\frac{A_{i^*,1}}{A_{i^*,2}}} &\leq \frac{1-c}{1-\frac{\widehat{A}_{i^*,1}}{\widehat{A}_{i^*,2}}}-\frac{1-c}{1-\frac{A_{i^*,1}}{A_{i^*,2}}} \ \tag{From the definition of $i'$ and $\widebar{F_\delta}$}\\
    &\leq (1-c) \left(\frac{1}{1-\frac{\widehat{A}_{i^*,1}}{\widehat{A}_{i^*,2}}}-\frac{1}{1-\frac{A_{i^*,1}}{A_{i^*,2}}}\right)
\end{align*}
The first inequality follows by noting that under $\widebar{F_\delta},\ i^*\in \{i: \widehat{A}_{i, 2}>\widehat{A}_{i, 1}\}$ and then we use the fact that $i'$ is the argmin obtained after searching over the set $\{i: \widehat{A}_{i, 2}>\widehat{A}_{i, 1}\}$. Now, using $\frac{A_{i^*, 1}}{A_{i^*, 2}}<c$, we have the following with probability $1-\mathbb{P}(F_\delta)$,
\begin{align*}
    x^t-x^* &\leq (1-c) \left(\frac{1}{1-\frac{\widehat{A}_{i^*,1}}{\widehat{A}_{i^*,2}}}-\frac{1}{1-\frac{\widehat{A}_{i^*, 1}}{\widehat{A}_{i^*, 2}} + \frac{2c\delta}{1+\delta}}\right)\\
    &\leq \frac{2c\delta}{(1+\delta)(1-c)}.
\end{align*}
\newline
\textbf{Case (II) $x^*<1, x^t=1$:} From Lemma~\ref{lem:lemOne}, we have that $|x^*-x^t|_+=0$.
For bounding, $x^t-x^*$, we first note that $x^*<1\implies \frac{A_{i^*, 1}}{A_{i^*, 2}}<c$. Also, $x^t=1 $ implies either $ \frac{\widehat{A}_{i', 1}}{\widehat{A}_{i', 2}}\geq c$ or $\{i:\widehat{A}_{i', 1}<\widehat{A}_{i', 2}\}=\phi$.

\begin{align*}
    x^t-x^* = \frac{c-\frac{A_{i^*, 1}}{A_{i^*, 2}}}{1-\frac{A_{i^*, 1}}{A_{i^*, 2}}} 
\end{align*}

Consider events $G = \{i:\widehat{A}_{i, 2}> \widehat{A}_{i, 1}\}\neq \phi\}$ and $F_\delta$. As RHS is a decreasing function in $\frac{\widehat{A}_{i^*, 1}}{\widehat{A}_{i^*, 2}}$, with probability $\left(1-\mathbb{P}(F_\delta)\right)\mathbb{P}(G)$, we have the following.
\begin{align*}
    x^t-x^* \leq \frac{c-\frac{\widehat{A}_{i^*, 1}}{\widehat{A}_{i^*, 2}} + \frac{2c\delta}{1+\delta}}{1-\frac{\widehat{A}_{i^*, 1}}{\widehat{A}_{i^*, 2}} + \frac{2c\delta}{1+\delta}}
    \leq \frac{c-\frac{\widehat{A}_{i', 1}}{\widehat{A}_{i', 2}} + \frac{2c\delta}{1+\delta}}{1-\frac{\widehat{A}_{i', 1}}{\widehat{A}_{i', 2}} + \frac{2c\delta}{1+\delta}} \leq \frac{2c\delta}{(1-c)(1+\delta)}
\end{align*}
where the second inequality follows from the definition of $i'$ which implies $\frac{\widehat{A}_{i^*, 1}}{\widehat{A}_{i^*, 2}}\geq \frac{\widehat{A}_{i', 1}}{\widehat{A}_{i', 2}}$ and the last inequality uses that $\frac{\widehat{A}_{i', 1}}{\widehat{A}_{i', 2}}\geq c$.

Now, for the event $G$, we will lower bound $\mathbb{P}(\widehat{A}_{i,1}-\widehat{A}_{i,2}> 0)$ for an $i\in [n]$ for which $A_{i,2}> A_{i,1}$. Notice that such an index always exists as $[n]\setminus [n_1]$ is considered non-empty (to avoid a trivial policy).
\begin{align*}
\mathbb{P}(\widehat{A}_{i,1}-\widehat{A}_{i,2}> 0) &= \mathbb{P}\left((\widehat{A}_{i,1}- A_{i,1})-(\widehat{A}_{i,2}-A_{i,2})> A_{i,2}-A_{i,1}\right)\\
&= \mathbb{P}\left(e^{\lambda(\widehat{A}_{i,1}- A_{i,1})-\lambda(\widehat{A}_{i,2}-A_{i,2})}> e^{\lambda(A_{i,2}-A_{i,1})}\right) \ \tag{for $\lambda\in \mathbb{R}, \
\lambda>0$}\\
&\leq \mathbb{E}\left[ e^{\lambda(\widehat{A}_{i,1}- A_{i,1})-\lambda(\widehat{A}_{i,2}-A_{i,2})} \right]e^{\lambda(A_{i,1}-A_{i,2})} \ \tag{Markov's inequality} \\
&\leq \exp\left({\frac{\lambda^2\sigma^2}{t'}}\right) \exp\left({\lambda(A_{i,1}-A_{i,2})}\right). \ \tag{Subgaussian's property}
\end{align*}
In the last inequality, we have used that as the rewards $X_{i, j}$ are assumed to be $\sigma$-subgaussian, each of $\widehat{A}_{i,1}- A_{i,1}$ and $\widehat{A}_{i,2}-A_{i,2}$ will be 0-mean $\frac{\sigma}{\sqrt{t}}$-subgaussians. Also, as the difference of two subgaussian random variables is a subgaussian random variable, $\widehat{A}_{i,1}- A_{i,1}-(\widehat{A}_{i,2}-A_{i,2})$ is a 0-mean subgaussian with parameter $\frac{\sqrt{2}\sigma}{\sqrt{t}}$.
On choosing $\lambda=\frac{(A_{i,2}-A_{i,1})t'}{2\sigma^2}$, we obtain the tightest upper-bound on the RHS, giving $ \mathbb{P}(\widehat{A}_{i,1}-\widehat{A}_{i,2}> 0) \leq \exp\left(-\frac{(A_{i,2}-A_{i,1})^2t'}{4\sigma^2}\right)\leq \exp\left(-\frac{(a_{\min}-1)^2t'}{4\sigma^2}\right).$ Thus, $\mathbb{P}(G)\geq 1-\exp\left(-\frac{(a_{\min}-1)^2t'}{4\sigma^2}\right)$.
\newline
\textbf{Case (III) $x^*=1, x^t<1$:} Here, $|x^t-x^*|_+=0$ as $x^t-x^*<0$ and from Lemma~\ref{lem:lemOne}, we have that with probability atleast $\mathbb{P}(\widebar{E_\delta})$, $|x^*-x^t|_+\leq \frac{2c\delta}{1+\delta}$.
\newline
\newline
Thus, combining all the cases, the expected fairness regret of \textsc{Explore-First} algorithm is given as follows.
\begin{align*}
    \mathbb{E}\left[ \mathcal{R}_{\textsc{FR}}(T) \right]&\leq (T-2t')\Bigg[n\left( \mathbb{P}(\widebar{E_\delta})\frac{2c\delta}{(1-c)(1+\delta)} + \mathbb{P}(E_\delta) \right) +\\
    &\quad (n-n_1) \left(\max\left\{ \mathbb{P}(\widebar{F_\delta})\mathbb{P}(G)\frac{2c\delta}{(1-c)(1+\delta)} + 1-\mathbb{P}(\widebar{F_\delta})\mathbb{P}(G),\   
    \mathbb{P}(\widebar{F_\delta})\frac{2c\delta}{(1-c)(1+\delta)} + \mathbb{P}(F_\delta)
    \right\} \right)\Bigg]\\
    &\leq (T-2t')\Bigg[n\left( \mathbb{P}(\widebar{E_\delta})\frac{2c\delta}{(1-c)(1+\delta)} + \mathbb{P}(E_\delta) \right) + \\
    &\quad (n-n_1)\left( \mathbb{P}(\widebar{F_\delta})\mathbb{P}(G)\frac{2c\delta}{(1-c)(1+\delta)} + 1-\mathbb{P}(\widebar{F_\delta})\mathbb{P}(G)+\   
    \mathbb{P}(\widebar{F_\delta})\frac{2c\delta}{(1-c)(1+\delta)} + \mathbb{P}(F_\delta)
     \right)\Bigg]\\
    &\leq (T-2t')\Bigg[n\left( \mathbb{P}(\widebar{E_\delta})\frac{2c\delta}{(1-c)(1+\delta)} + \mathbb{P}(E_\delta) \right) + \\
    &\quad (n-n_1)\left( \mathbb{P}(\widebar{F_\delta})\frac{2c\delta}{(1-c)(1+\delta)} + 1-\mathbb{P}(\widebar{F_\delta})\mathbb{P}(G)+\   
    \mathbb{P}(\widebar{F_\delta})\frac{2c\delta}{(1-c)(1+\delta)} + \mathbb{P}(F_\delta)
     \right)\Bigg]\\
     &\leq (T-2t')\Bigg[n\left( \mathbb{P}(\widebar{E_\delta})\frac{2c\delta}{(1-c)(1+\delta)} + \mathbb{P}(E_\delta) \right)+\\&\quad (n-n_1)\left\{\mathbb{P}(\widebar{F_\delta})\left(\frac{4c\delta}{(1-c)(1+\delta)}+\exp{\left(-\frac{(a_{\min}-1)^2t'}{4\sigma^2}\right)}\right) + 2\mathbb{P}(F_\delta)\right\}\Bigg]\\
     &\leq (T-2t')\Bigg[6n\left( 1-2\exp{\left(-\frac{t'}{2\sigma^2}\delta^2a^2_{\min}\right)}\right)\left(\frac{c\delta }{(1-c)(1+\delta)}+\exp{\left(-\frac{(a_{\min}-1)^2t'}{4\sigma^2}\right)}\right)+ \exp{\left(-\frac{t'}{2\sigma^2}\delta^2a^2_{\min}\right)} \Bigg]\\
     &\textup{Choose }2t'=T^{2/3} \textup{ and }\delta = \left(\frac{2\sigma}{ a_{\min}}\right)\sqrt{\frac{\log{T}}{3T^{2/3}}} \textup{ with }T \textup{ large enough such that }\delta\in(0, 1) \textup{, to get}\\
     & \leq 6nT^{2/3}\left( \frac{2c\sigma \sqrt{\log{T}}}{(1-c)a_{\min}} +\exp\left(-\frac{(a_{\min}-1)^2T^{2/3}}{8\sigma^2}\right)+1\right)\\
     & \leq 6nT^{2/3}\left( \frac{2\sigma \sqrt{\log{T}}}{a_{\min}} +\exp\left(-\frac{(a_{\min}-1)^2T^{2/3}}{8\sigma^2}\right)+1\right) \tag{with the sufficient condition for feasibility, $c\in (0, 1/2]$}\\
     &\leq 24nT^{2/3}\frac{\sigma \sqrt{\log{T}}}{a_{\min}}.
\end{align*}
This proves the claim. Our proof needs $c<1$, which is justified because $c=1$ would simplify the corresponding fairness constraint to be $ \langle A_i,\pi\rangle = \Amax_i$, which in turn specifies fixes coordinates of $\pi$ making the problem simpler.
\end{proof}

%% file: AAMAS25/ucb_fr.tex
\gan{choose $\delta  =   1/ T$}

\gan{@Himanshu: Go over the proof once and write your comments. We will meet on Friday to discuss all the proofs in the paper.}
\fairnessRegretUCB*
\gan{Need to be careful. What is $\varepsilon$ above. Also, write the entire expression in terms of expected regret and not the whp regret. We can do this by $regret \leq  (1 - \delta) \tilde{O}(T^{3/4}) + \delta T $ and then optimize over delta.  }
\begin{proof}
The proof uses the notation $\epsilon_{j}^t=\frac{\epsilon}{\sqrt{N_j^t}}=\sigma \sqrt{\frac{2\log{(8mn/\delta)}}{N_j^t}} $. Also, $\epsilon^t\in \mathbb{R}^m$ with the $j^{th}$ entry as $\epsilon^t_j$. Let $j_i\in \argmax_{j\in[m]} A_{i, j}$ and $j_i^t\in \argmax_{j\in[m]}\underline{A}_{i, j}^t$. 
We use $\bar{C}:=\sum_{i\in [n]}C_i$ and $C_{\max}:=\max_{i\in [n]}C_i$.
\begin{align*}
\mathcal{R}_{FR}(T) &= \sum_{t=1}^T \sum_{i=1}^n \big | C_i \Amax_i  -   \langle A_i, \pi^t \rangle  \big|_{+} \\ 
& = \sum_{t=1}^{m\lceil \sqrt{T } \rceil } \sum_{i=1}^n \big | C_i \Amax_i  -   \langle A_i, \pi^t \rangle  \big|_{+} +  \sum_{t= m\lceil \sqrt{T } \rceil + 1}^T \sum_{i=1}^n \big | C_i \Amax_i  -   \langle A_i, \pi^t \rangle  \big|_{+} \\
& \leq \sum_{t=1}^{m\lceil \sqrt{T } \rceil } \sum_{i=1}^n \big | C_i \Amax_i  -   \langle A_i, \pi^t \rangle \big|_{+} + \sum_{t= m\lceil \sqrt{T } \rceil +1 }^T \sum_{i=1}^n \big | C_i \Amax_i  -    \langle \overline{A}_i^t   -   2\varepsilon^{t},  \pi^t \rangle |_+ \tag{with probability at least $1-\delta/2$}\\ 
&\leq \bar{C} m\sqrt{T}  + \sum_{t=m\lceil \sqrt{T } \rceil +1}^T \sum_{i=1}^n \big | C_i \Amax_i  -   \langle \overline{A}_i^t, \pi^t \rangle |_+  +   2n \varepsilon \sum_{t=m\lceil \sqrt{T } \rceil +1}^T \mathbb{E}_{j \sim \pi^t} \big [ \sqrt{1/N_{j}^{t}} \big ] \\ 
&\leq  \bar{C}m\sqrt{T} + \sum_{t=m\lceil \sqrt{T } \rceil +1}^T \sum_{i=1}^n C_i \big  |  \Amax_i  -   \underline{A}_{i, j_i^t}^t|_+  +   2n \varepsilon \sum_{t=m\lceil \sqrt{T } \rceil +1}^T \mathbb{E}_{j \sim \pi^t} \Big [\sqrt{1/N_{j}^{t}} \Big ] \tag{follows from the fact that $\pi^t$ is feasible for P2; i.e. $ \langle \overline{A}_i^t, \pi^t \rangle \geq C_i \underline{A}_{i, j_i^t}^t $}  \\ 
& \leq \bar{C}m\sqrt{T} +   C_{\max}\sum_{t= m\lceil \sqrt{T } \rceil +1 }^T \sum_{i=1}^n |A_{i, j_i} - \underline{A}_{i, j_i}^t|_{+} +  2n \varepsilon \sum_{t=m\lceil \sqrt{T } \rceil +1}^T \mathbb{E}_{j \sim \pi^t} \Big [\sqrt{1/N_{j}^{t}} \Big ]
\tag{follows from the fact that $j_i^t \in \arg\max_{j \in [m]} \underline{A}_{i,j}^t$} \\ 
& \leq \bar{C}m \sqrt{T} + 2C_{\max}  \sum_{t=m\lceil \sqrt{T } \rceil + 1}^T \sum_{i=1}^n \varepsilon_{j_i}^t +  2 n \varepsilon \sum_{t=m\lceil \sqrt{T } \rceil +1}^T \mathbb{E}_{j \sim \pi^t} \Big [\sqrt{1/N_{j}^{t}} \Big ] \\
& \leq \bar{C}m \sqrt{T} + 2C_{\max}  \sum_{t=m\lceil \sqrt{T } \rceil + 1}^T \sum_{i=1}^n \varepsilon_{j_i}^t +  2 n \varepsilon \sum_{t=m +1}^T \mathbb{E}_{j \sim \pi^t} \Big [\sqrt{1/N_{j}^{t}} \Big ]\\
& \leq   \bar{C}m \sqrt{T }  + 2 C_{\max} n \varepsilon  \sum_{t= 1 }^T \frac{1}{\sqrt{N_{j_i}^{t}}} +  2n\varepsilon  ( \sqrt{T} + \sqrt{2T\log(2/\delta)}) \tag{Using Lemma  \ref{lemma:wang21}; with probability $1-\delta/2$}\\
& \leq   \bar{C}m \sqrt{T }  + 2 C_{\max} n \varepsilon  \sum_{t=  1}^T \frac{1}{\sqrt{ \sqrt{T} }} +   2n\varepsilon ( \sqrt{T} + \sqrt{2T\log(2/\delta)}) \ \tag{$\because N_j^t\geq \sqrt{T}\ \forall j\in [m]$} \\
& \leq   \bar{C}m \sqrt{T }  + 4nC_{\max}\sigma \sqrt{\log(8mn/\delta)} \left(   T^{3/4} +   ( \sqrt{T} + \sqrt{2T\log(2/\delta)})\right).
\end{align*}
Hence, the expected fairness regret is upper-bounded as follows.
\begin{align*}
\mathcal{R}_{\textsc{FR}}(T) & = (1-\delta)\left( \bar{C}m \sqrt{T }  + 4n\sigma C_{\max} \sqrt{\log(8mn/\delta)} \left(  T^{3/4} +   ( \sqrt{T} + \sqrt{2T\log(2/\delta)})\right) \right) + \delta T    \\
&\leq \bar{C}m\sqrt{T}+4n\sigma C_{\max}\left( \sqrt{\log(8mn)/\delta}\ T^{3/4} +  \sqrt{T\log(8mn/\delta)} +  \log(8mn/\delta)T^{3/4}\right) +\delta T\\
&\leq \bar{C}m\sqrt{T} + 12n\sigma C_{\max} T^{3/4}\log{(8mn/\delta)} + \delta T\\
&\leq 6 C_{\max}n\sigma T^{3/4}\log(2m^2T)+mnO(\sqrt{T}) \tag{with $\delta = 4\sqrt{2}n\sigma/\sqrt{T}$}
\end{align*}

\end{proof}

%% file: AAMAS25/simulation_Appendix.tex
\section{Additional Experimental Results}
In addition to the results in Sec~\ref{sec: simulation}, we show additional simulations done with different $A$ matrices, shared by~\cite{Hossain2020FairAF} in their code open-sourced repository, and different values of $C_i=c\ \forall i\in [n]$. We plot the average regrets after simulating with 100 runs for different realizations of randomly sampled rewards. The codes for reproducing our experiments are \href{https://github.com/Piyushi-0/Fair-MAMAB}{open-sourced}.

\subsection{Exploration-Exploitation Tradeoff in \textsc{Explore-First}}\label{app:ef}
This section shows how the exploration parameter $\alpha$ trades off between social welfare regret and fairness regret across different instances with simulated data. Supporting our theoretical result from the analysis of the 2-arm case, $\alpha=0.67$ can be seen to obtain both optimal social welfare regret and fairness regret across different instances shown in Figures~\ref{fig:ef_n3m2} and \ref{fig:ef_n4m3}.
\begin{figure*}[ht!]
\centering
\begin{subfigure}{.24\textwidth}
    \centering
    \includegraphics[width=\linewidth]{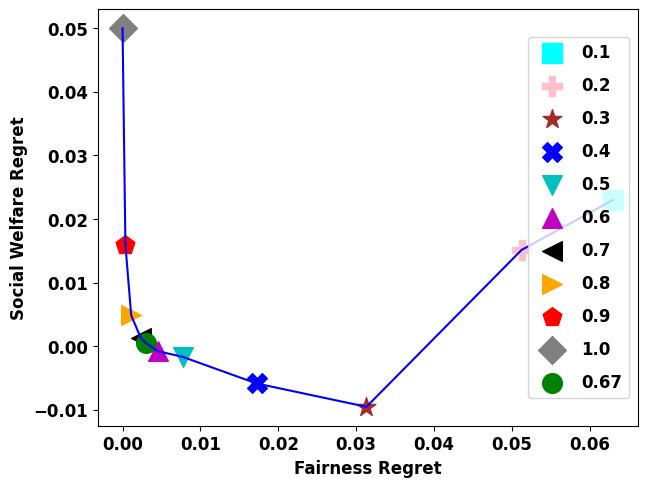}  
    \label{ef3-1}
\end{subfigure}
\begin{subfigure}{.24\textwidth}
    \centering
    \includegraphics[width=\linewidth]{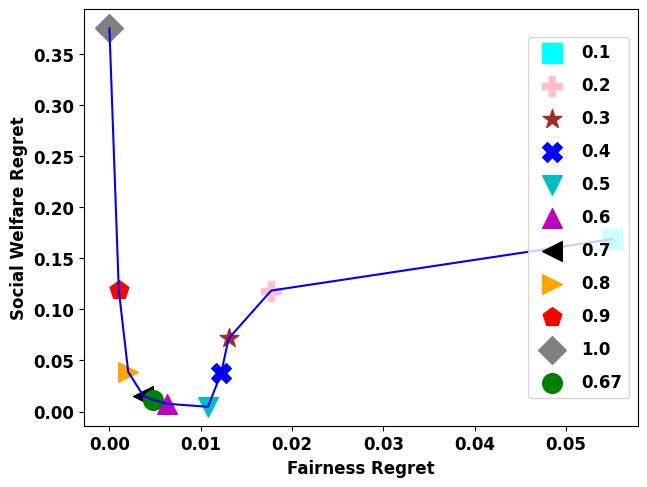}  
    \label{ef3-2}
\end{subfigure}
\begin{subfigure}{.24\textwidth}
    \centering
    \includegraphics[width=\linewidth]{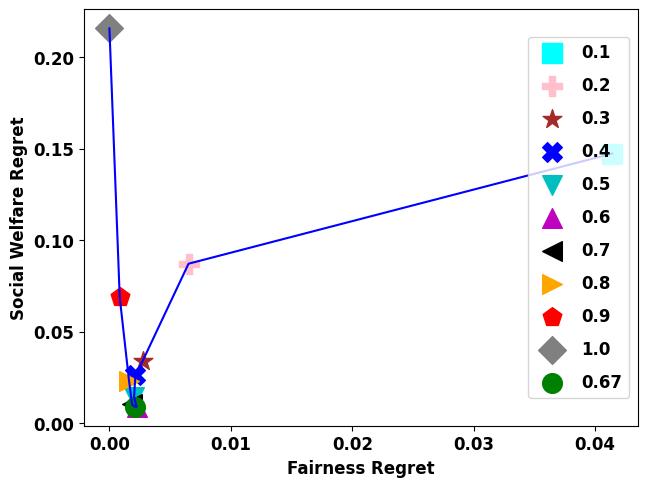} 
    \label{ef3-3}
\end{subfigure}
\begin{subfigure}{.24\textwidth}
    \centering
    \includegraphics[width=\linewidth]{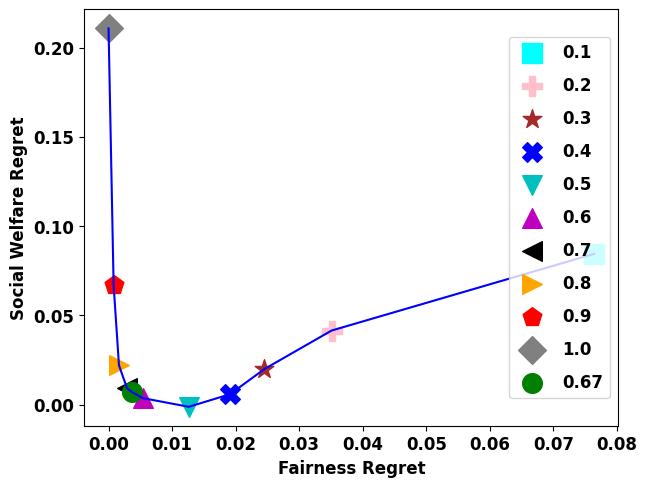} 
    \label{ef3-4}
\end{subfigure} 
\vspace{-0.15in}
    \caption{Explore-Exploit tradeoff on varying the exploration hyperparameter $\alpha$ (marked in legend). Different plots show different $A$ matrices with $n=3, m=2$. $C_i$ is $1/m\ \forall i \in [n]$. The values reported are averaged across 100 random runs with random seeds.}
    \label{fig:ef_n3m2}
    \Description{Exploration-exploitation trade-off with Explore-First algorithm in the two-arm case.}
\end{figure*} 

\begin{figure*}[ht!]
\centering
\begin{subfigure}{.24\textwidth}
    \centering
    \includegraphics[width=\linewidth]{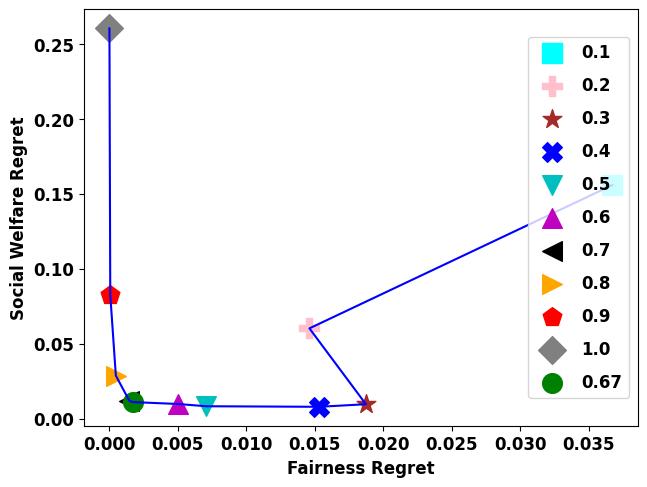}  
    \label{ef4-1}
\end{subfigure}
\begin{subfigure}{.24\textwidth}
    \centering
    \includegraphics[width=\linewidth]{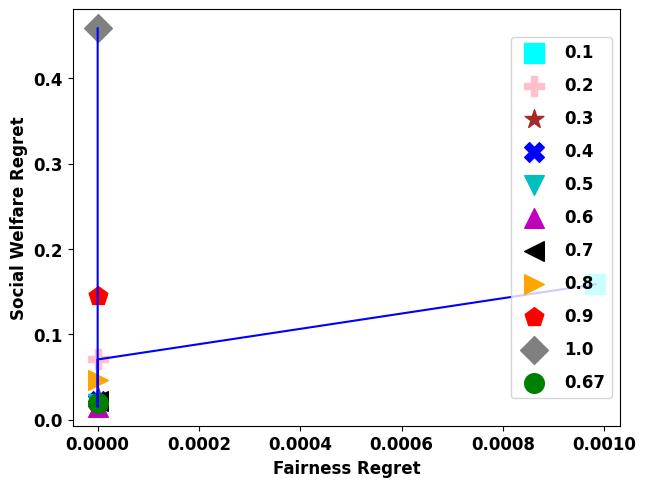}  
    \label{ef4-2}
\end{subfigure}
\begin{subfigure}{.24\textwidth}
    \centering
    \includegraphics[width=\linewidth]{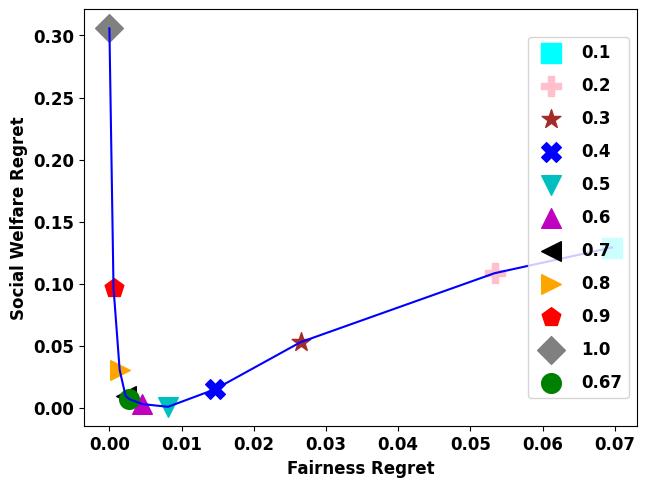} 
    \label{ef4-3}
\end{subfigure}
\begin{subfigure}{.24\textwidth}
    \centering
    \includegraphics[width=\linewidth]{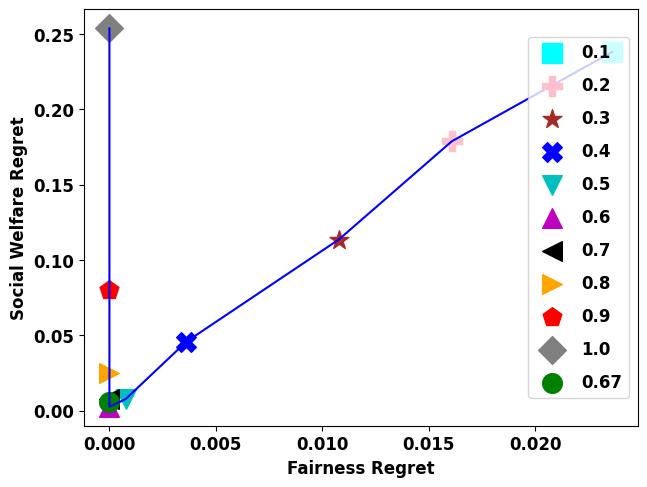} 
    \label{ef4-4}
    \Description{Exploration-exploitation trade-off with Explore-First algorithm in the three-arm case.}
\end{subfigure}
\vspace{-0.15in}
    \caption{Explore-Exploit tradeoff on varying the exploration hyperparameter $\alpha$ (marked in legend). Different plots show different $A$ matrices with $n=4, m=3$. $C_i$ is $1/m\ \forall i \in [n]$.}
    \label{fig:ef_n4m3}
\end{figure*} 

\subsection{Evaluation of \ouralgo\ }
Figures \ref{app:sw_32_0.5} and \ref{app:fr_32_0.5} compares the regrets incurred by the proposed \ouralgo\, \textsc{Explore-First} (Algo~\ref{algOne-EF}) and the dual heuristics (Algo~\ref{alg-dual}). The results are shown with different $A$ matrices with $n=3$, $m=2$. Figures \ref{app:sw_43_0.3} and \ref{app:fr_43_0.3} compares the regrets incurred by \ouralgo\ with the baselines \textsc{Explore-First} (Algo~\ref{algOne-EF}) and the dual heuristics (Algo~\ref{alg-dual}). The results are shown with different $A$ matrices with $n=4$, $m=3$. We can see that \ouralgo\ obtains optimal social welfare regret and sublinear fairness regret across the instances.

\begin{figure*}[ht!]
\centering
\begin{subfigure}{.33\textwidth}
    \centering
    \includegraphics[width=\linewidth]{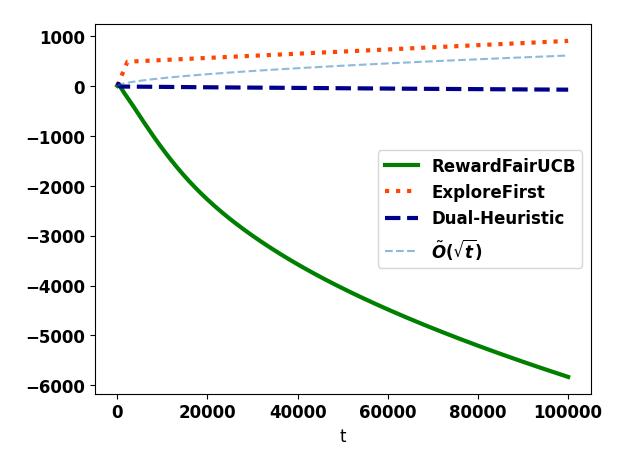}  
\end{subfigure}  
\begin{subfigure}{.33\textwidth}
    \centering
    \includegraphics[width=\linewidth]{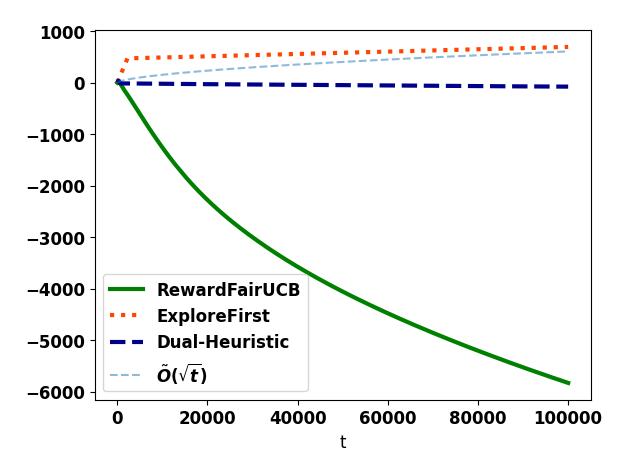}  
\end{subfigure}  
\begin{subfigure}{.33\textwidth}
    \centering
    \includegraphics[width=\linewidth]{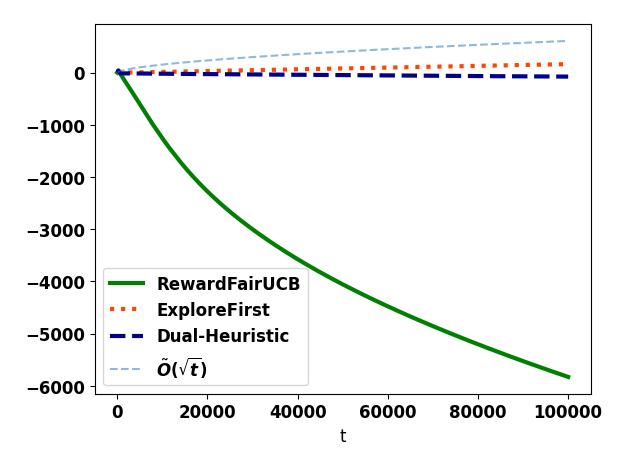}  
\end{subfigure}  
\caption{Social welfare regrets vs timesteps with different $A$ matrices with $n=3$, $m=2$ and $C_i=0.5\ \forall i \in [n]$.}\label{app:sw_32_0.5}
\Description{Social welfare regret two-arm case.}
\end{figure*} 

\begin{figure*}[ht!]
\centering
\begin{subfigure}{.33\textwidth}
    \centering
    \includegraphics[width=\linewidth]{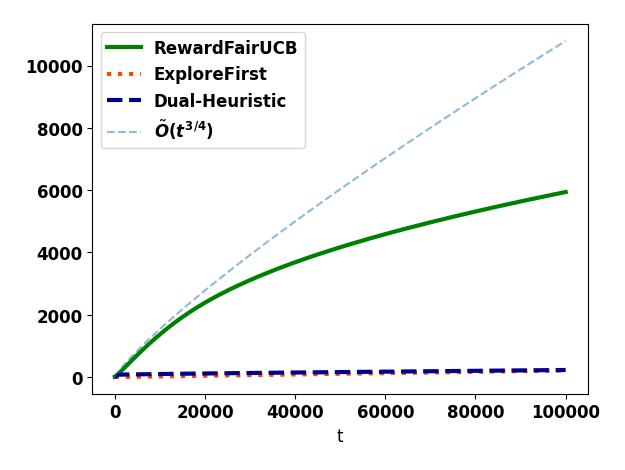}  
\end{subfigure}  
\begin{subfigure}{.33\textwidth}
    \centering
    \includegraphics[width=\linewidth]{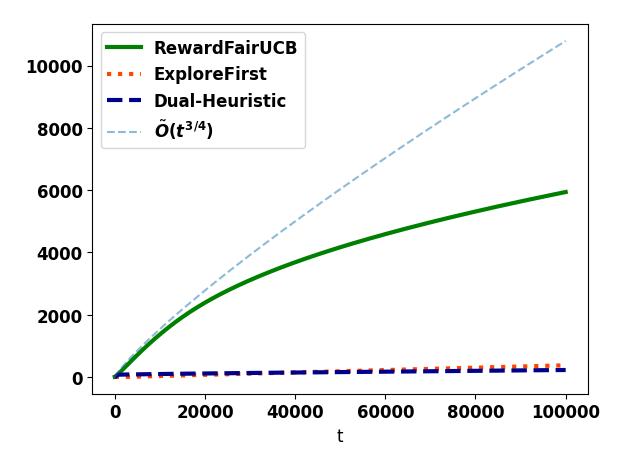}  
\end{subfigure}  
\begin{subfigure}{.33\textwidth}
    \centering
    \includegraphics[width=\linewidth]{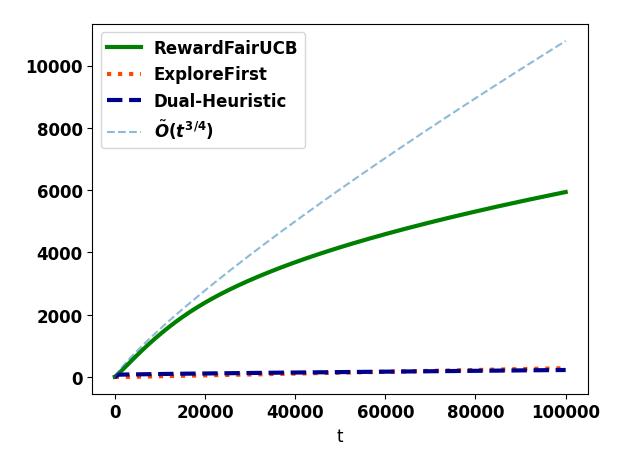}  
\end{subfigure}  
\caption{Fairness regrets vs timesteps with with different $A$ matrices with $n=3$, $m=2$ and $C_i=0.5\ \forall i \in [n]$.}\label{app:fr_32_0.5}
\Description{Fairness regret two-arm case.}
\end{figure*} 

\begin{figure*}[ht!]
\centering
\begin{subfigure}{.33\textwidth}
    \centering
    \includegraphics[width=\linewidth]{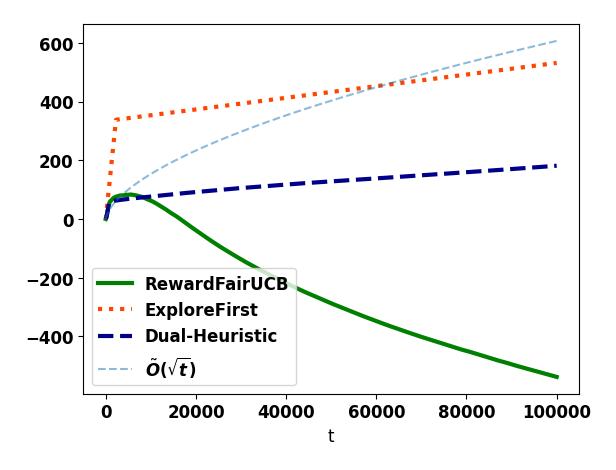}  
\end{subfigure}  
\begin{subfigure}{.33\textwidth}
    \centering
    \includegraphics[width=\linewidth]{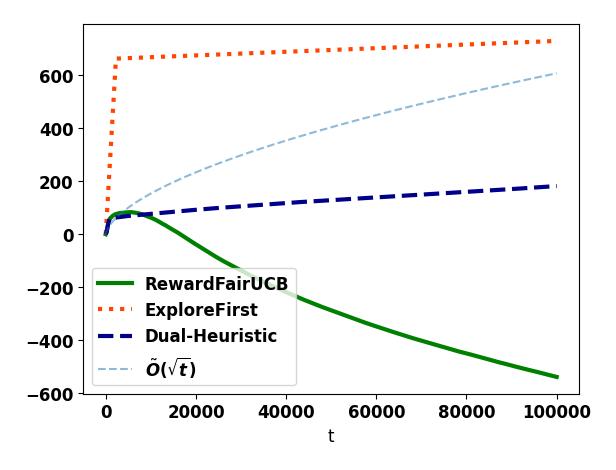}  
\end{subfigure}  
\begin{subfigure}{.33\textwidth}
    \centering
    \includegraphics[width=\linewidth]{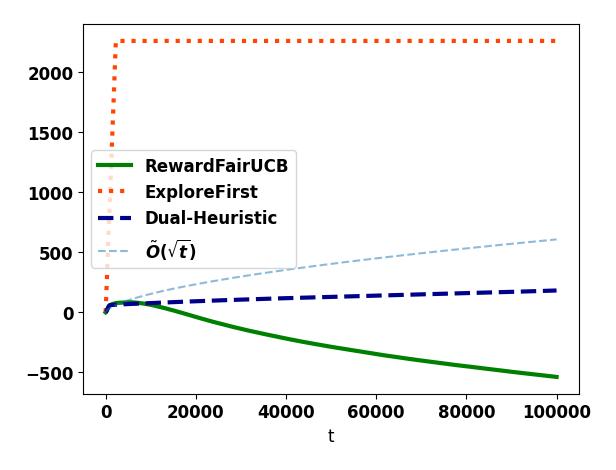}  
\end{subfigure}  
\caption{Social welfare regrets vs timesteps with with different $A$ matrices with $n=4$, $m=3$ and $C_i=0.3\ \forall i \in [n]$.}\label{app:sw_43_0.3}
\Description{Social welfare regret three-arm case.}
\end{figure*}

\begin{figure*}[ht!]
\centering
\begin{subfigure}{.33\textwidth}
    \centering
    \includegraphics[width=\linewidth]{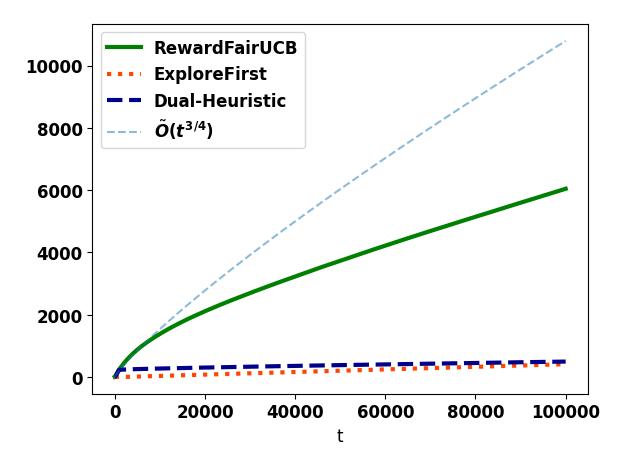}  
\end{subfigure}  
\begin{subfigure}{.33\textwidth}
    \centering
    \includegraphics[width=\linewidth]{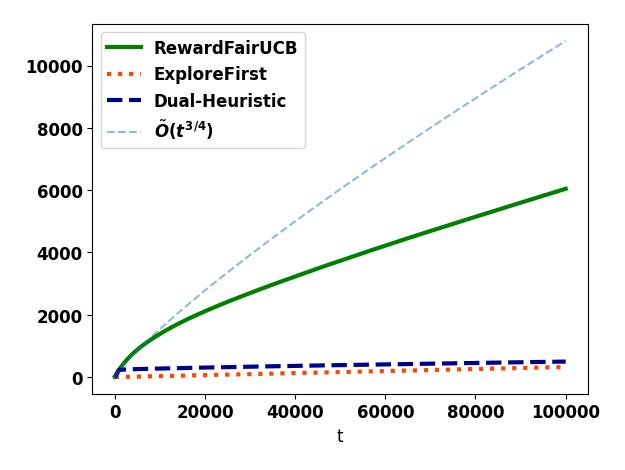}  
\end{subfigure}  
\begin{subfigure}{.33\textwidth}
    \centering
    \includegraphics[width=\linewidth]{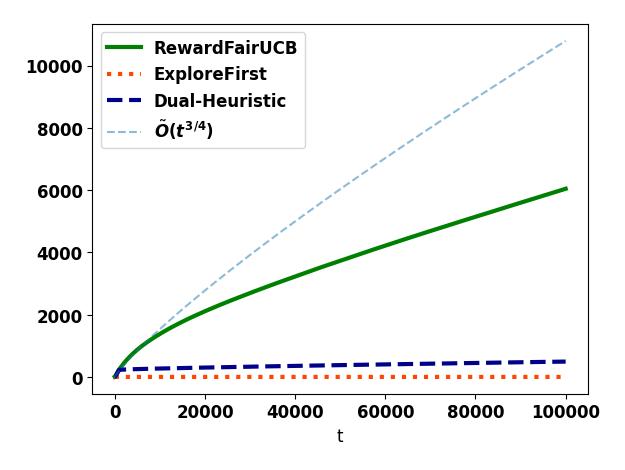}  
\end{subfigure}  
\caption{Fairness regrets vs timesteps with with different $A$ matrices with $n=4$, $m=3$ and $C_i=0.3\ \forall i \in [n]$.}\label{app:fr_43_0.3}
\Description{Fairness regret three-arm case.}
\end{figure*}

%% file: main.bbl

\begin{thebibliography}{22}


\ifx \showCODEN    \undefined \def \showCODEN     #1{\unskip}     \fi
\ifx \showDOI      \undefined \def \showDOI       #1{#1}\fi
\ifx \showISBNx    \undefined \def \showISBNx     #1{\unskip}     \fi
\ifx \showISBNxiii \undefined \def \showISBNxiii  #1{\unskip}     \fi
\ifx \showISSN     \undefined \def \showISSN      #1{\unskip}     \fi
\ifx \showLCCN     \undefined \def \showLCCN      #1{\unskip}     \fi
\ifx \shownote     \undefined \def \shownote      #1{#1}          \fi
\ifx \showarticletitle \undefined \def \showarticletitle #1{#1}   \fi
\ifx \showURL      \undefined \def \showURL       {\relax}        \fi
\providecommand\bibfield[2]{#2}
\providecommand\bibinfo[2]{#2}
\providecommand\natexlab[1]{#1}
\providecommand\showeprint[2][]{arXiv:#2}

\bibitem[\protect\citeauthoryear{Auer, Cesa-Bianchi, Freund, and Schapire}{Auer et~al\mbox{.}}{2002}]%
        {Auer02}
\bibfield{author}{\bibinfo{person}{Peter Auer}, \bibinfo{person}{Nicol\`{o} Cesa-Bianchi}, \bibinfo{person}{Yoav Freund}, {and} \bibinfo{person}{Robert~E. Schapire}.} \bibinfo{year}{2002}\natexlab{}.
\newblock \showarticletitle{The Nonstochastic Multiarmed Bandit Problem}.
\newblock \bibinfo{journal}{\emph{Society for Industrial and Applied Mathematics (SIAM) Journal on Computing}} \bibinfo{volume}{32}, \bibinfo{number}{1} (\bibinfo{year}{2002}), \bibinfo{pages}{48--77}.
\newblock


\bibitem[\protect\citeauthoryear{Barman, Khan, Maiti, and Sawarni}{Barman et~al\mbox{.}}{2023}]%
        {Barman_Khan_Maiti_Sawarni_2023}
\bibfield{author}{\bibinfo{person}{Siddharth Barman}, \bibinfo{person}{Arindam Khan}, \bibinfo{person}{Arnab Maiti}, {and} \bibinfo{person}{Ayush Sawarni}.} \bibinfo{year}{2023}\natexlab{}.
\newblock \showarticletitle{Fairness and welfare quantification for regret in multi-armed bandits}. In \bibinfo{booktitle}{\emph{Association for the Advancement of Artificial Intelligence (AAAI)}} \emph{(\bibinfo{series}{AAAI'23})}. \bibinfo{publisher}{AAAI Press}, \bibinfo{pages}{6762--6769}.
\newblock


\bibitem[\protect\citeauthoryear{Baudry, Merlis, Benjamin~Molina, Richard, and Perchet}{Baudry et~al\mbox{.}}{2024}]%
        {baudry2024}
\bibfield{author}{\bibinfo{person}{Dorian Baudry}, \bibinfo{person}{Nadav Merlis}, \bibinfo{person}{Mathieu Benjamin~Molina}, \bibinfo{person}{Hugo Richard}, {and} \bibinfo{person}{Vianney Perchet}.} \bibinfo{year}{2024}\natexlab{}.
\newblock \showarticletitle{Multi-armed bandits with guaranteed revenue per arm}. In \bibinfo{booktitle}{\emph{AISTATS}} \emph{(\bibinfo{series}{Proceedings of Machine Learning Research}, Vol.~\bibinfo{volume}{238})}. \bibinfo{publisher}{PMLR}, \bibinfo{pages}{379--387}.
\newblock


\bibitem[\protect\citeauthoryear{Besson and Kaufmann}{Besson and Kaufmann}{2018}]%
        {besson2018doubling}
\bibfield{author}{\bibinfo{person}{Lilian Besson} {and} \bibinfo{person}{Emilie Kaufmann}.} \bibinfo{year}{2018}\natexlab{}.
\newblock \showarticletitle{What Doubling Tricks Can and Can't Do for Multi-Armed Bandits}.
\newblock \bibinfo{journal}{\emph{arXiv preprint arXiv:1803.06971}} (\bibinfo{year}{2018}).
\newblock


\bibitem[\protect\citeauthoryear{Bubeck, Cesa-Bianchi, et~al\mbox{.}}{Bubeck et~al\mbox{.}}{2012}]%
        {bubeck2012regretanalysisstochasticnonstochastic}
\bibfield{author}{\bibinfo{person}{S{\'e}bastien Bubeck}, \bibinfo{person}{Nicolo Cesa-Bianchi}, {et~al\mbox{.}}} \bibinfo{year}{2012}\natexlab{}.
\newblock \bibinfo{booktitle}{\emph{Regret analysis of stochastic and nonstochastic multi-armed bandit problems}}. Vol.~\bibinfo{volume}{5}.
\newblock \bibinfo{publisher}{Now Publishers, Inc.} 1--122 pages.
\newblock


\bibitem[\protect\citeauthoryear{Caragiannis, Kurokawa, Moulin, Procaccia, Shah, and Wang}{Caragiannis et~al\mbox{.}}{2019}]%
        {caragiannis2019unreasonable}
\bibfield{author}{\bibinfo{person}{Ioannis Caragiannis}, \bibinfo{person}{David Kurokawa}, \bibinfo{person}{Herv{\'e} Moulin}, \bibinfo{person}{Ariel~D Procaccia}, \bibinfo{person}{Nisarg Shah}, {and} \bibinfo{person}{Junxing Wang}.} \bibinfo{year}{2019}\natexlab{}.
\newblock \showarticletitle{The unreasonable fairness of maximum Nash welfare}.
\newblock \bibinfo{journal}{\emph{ACM Transactions on Economics and Computation (TEAC)}} \bibinfo{volume}{7}, \bibinfo{number}{3} (\bibinfo{year}{2019}), \bibinfo{pages}{1--32}.
\newblock


\bibitem[\protect\citeauthoryear{Diamond and Boyd}{Diamond and Boyd}{2016}]%
        {diamond2016cvxpy}
\bibfield{author}{\bibinfo{person}{Steven Diamond} {and} \bibinfo{person}{Stephen Boyd}.} \bibinfo{year}{2016}\natexlab{}.
\newblock \showarticletitle{{CVXPY}: {A} {P}ython-embedded modeling language for convex optimization}.
\newblock \bibinfo{journal}{\emph{Journal of Machine Learning Research}} \bibinfo{volume}{17}, \bibinfo{number}{83} (\bibinfo{year}{2016}), \bibinfo{pages}{1--5}.
\newblock


\bibitem[\protect\citeauthoryear{Harper and Konstan}{Harper and Konstan}{2015}]%
        {10.1145/2827872}
\bibfield{author}{\bibinfo{person}{F.~Maxwell Harper} {and} \bibinfo{person}{Joseph~A. Konstan}.} \bibinfo{year}{2015}\natexlab{}.
\newblock \showarticletitle{The MovieLens Datasets: History and Context}.
\newblock \bibinfo{journal}{\emph{ACM Transactions on Interactive Intelligent Systems}} \bibinfo{volume}{5}, \bibinfo{number}{4}, Article \bibinfo{articleno}{19} (\bibinfo{date}{Dec} \bibinfo{year}{2015}), \bibinfo{numpages}{19}~pages.
\newblock
\showISSN{2160-6455}


\bibitem[\protect\citeauthoryear{Hossain, Micha, and Shah}{Hossain et~al\mbox{.}}{2021}]%
        {Hossain2020FairAF}
\bibfield{author}{\bibinfo{person}{Safwan Hossain}, \bibinfo{person}{Evi Micha}, {and} \bibinfo{person}{Nisarg Shah}.} \bibinfo{year}{2021}\natexlab{}.
\newblock \showarticletitle{Fair Algorithms for Multi-Agent Multi-Armed Bandits}. In \bibinfo{booktitle}{\emph{Neural Information Processing Systems (NeurIPS)}}, Vol.~\bibinfo{volume}{34}. \bibinfo{publisher}{Curran Associates, Inc.}, \bibinfo{pages}{24005--24017}.
\newblock


\bibitem[\protect\citeauthoryear{Jones, Nguyen, and Nguyen}{Jones et~al\mbox{.}}{2023}]%
        {Jones_Nguyen_Nguyen_2023}
\bibfield{author}{\bibinfo{person}{Matthew Jones}, \bibinfo{person}{Huy Nguyen}, {and} \bibinfo{person}{Thy Nguyen}.} \bibinfo{year}{2023}\natexlab{}.
\newblock \showarticletitle{An efficient algorithm for fair multi-agent multi-armed bandit with low regret}. In \bibinfo{booktitle}{\emph{Association for the Advancement of Artificial Intelligence (AAAI)}} \emph{(\bibinfo{series}{AAAI'23})}. \bibinfo{publisher}{AAAI Press}, Article \bibinfo{articleno}{916}, \bibinfo{numpages}{9}~pages.
\newblock


\bibitem[\protect\citeauthoryear{Kaneko and Nakamura}{Kaneko and Nakamura}{1979}]%
        {NSW}
\bibfield{author}{\bibinfo{person}{Mamoru Kaneko} {and} \bibinfo{person}{Kenjiro Nakamura}.} \bibinfo{year}{1979}\natexlab{}.
\newblock \showarticletitle{The Nash Social Welfare Function}.
\newblock \bibinfo{journal}{\emph{Econometrica}} \bibinfo{volume}{47}, \bibinfo{number}{2} (\bibinfo{year}{1979}), \bibinfo{pages}{423--435}.
\newblock
\showISSN{00129682, 14680262}


\bibitem[\protect\citeauthoryear{Krishna, John, Barik, and Tan}{Krishna et~al\mbox{.}}{2024}]%
        {krishna25pmean}
\bibfield{author}{\bibinfo{person}{Anand Krishna}, \bibinfo{person}{Philips~George John}, \bibinfo{person}{Adarsh Barik}, {and} \bibinfo{person}{Vincent Y.~F. Tan}.} \bibinfo{year}{2024}\natexlab{}.
\newblock \showarticletitle{p-Mean Regret for Stochastic Bandits}.
\newblock
\showeprint[arxiv]{2412.10751}~[cs.LG]


\bibitem[\protect\citeauthoryear{Lattimore and Szepesv{\'a}ri}{Lattimore and Szepesv{\'a}ri}{2020}]%
        {Lattimore2020BanditA}
\bibfield{author}{\bibinfo{person}{Tor Lattimore} {and} \bibinfo{person}{Csaba Szepesv{\'a}ri}.} \bibinfo{year}{2020}\natexlab{}.
\newblock \bibinfo{booktitle}{\emph{Bandit Algorithms}}.
\newblock \bibinfo{publisher}{Cambridge University Press}.
\newblock
\showISBNx{9781108571401}


\bibitem[\protect\citeauthoryear{Liu and Zhao}{Liu and Zhao}{2009}]%
        {Liu2009DistributedLI}
\bibfield{author}{\bibinfo{person}{Keqin Liu} {and} \bibinfo{person}{Qing Zhao}.} \bibinfo{year}{2009}\natexlab{}.
\newblock \showarticletitle{Distributed Learning in Multi-Armed Bandit With Multiple Players}.
\newblock \bibinfo{journal}{\emph{IEEE Transactions on Signal Processing}}  \bibinfo{volume}{58} (\bibinfo{year}{2009}), \bibinfo{pages}{5667--5681}.
\newblock


\bibitem[\protect\citeauthoryear{Liu, Radanovic, Dimitrakakis, Mandal, and Parkes}{Liu et~al\mbox{.}}{2017}]%
        {liu2017calibrated}
\bibfield{author}{\bibinfo{person}{Yang Liu}, \bibinfo{person}{Goran Radanovic}, \bibinfo{person}{Christos Dimitrakakis}, \bibinfo{person}{Debmalya Mandal}, {and} \bibinfo{person}{David~C Parkes}.} \bibinfo{year}{2017}\natexlab{}.
\newblock \showarticletitle{Calibrated fairness in bandits}.
\newblock \bibinfo{journal}{\emph{Proceedings of the 4th Workshop on Fairness, Accountability, and Transparency in Machine Learning (Fat/ML 2017)}} (\bibinfo{year}{2017}).
\newblock


\bibitem[\protect\citeauthoryear{Patil, Ghalme, Nair, and Narahari}{Patil et~al\mbox{.}}{2021}]%
        {JMLR:v22:20-704}
\bibfield{author}{\bibinfo{person}{Vishakha Patil}, \bibinfo{person}{Ganesh Ghalme}, \bibinfo{person}{Vineet Nair}, {and} \bibinfo{person}{Y. Narahari}.} \bibinfo{year}{2021}\natexlab{}.
\newblock \showarticletitle{Achieving Fairness in the Stochastic Multi-Armed Bandit Problem}.
\newblock \bibinfo{journal}{\emph{Journal of Machine Learning Research (JMLR)}} \bibinfo{volume}{22}, \bibinfo{number}{174} (\bibinfo{year}{2021}), \bibinfo{pages}{1--31}.
\newblock


\bibitem[\protect\citeauthoryear{Patil, Nair, Ghalme, and Khan}{Patil et~al\mbox{.}}{2023}]%
        {patil2022mitigating}
\bibfield{author}{\bibinfo{person}{Vishakha Patil}, \bibinfo{person}{Vineet Nair}, \bibinfo{person}{Ganesh Ghalme}, {and} \bibinfo{person}{Arindam Khan}.} \bibinfo{year}{2023}\natexlab{}.
\newblock \showarticletitle{Mitigating Disparity while Maximizing Reward: Tight Anytime Guarantee for Improving Bandits}. In \bibinfo{booktitle}{\emph{Proceedings of the Thirty-Second International Joint Conference on Artificial Intelligence, {IJCAI-23}}}. \bibinfo{publisher}{International Joint Conferences on Artificial Intelligence Organization}, \bibinfo{pages}{4100--4108}.
\newblock


\bibitem[\protect\citeauthoryear{Ron, Ben-Porat, and Shalit}{Ron et~al\mbox{.}}{2021}]%
        {porat21}
\bibfield{author}{\bibinfo{person}{Tom Ron}, \bibinfo{person}{Omer Ben-Porat}, {and} \bibinfo{person}{Uri Shalit}.} \bibinfo{year}{2021}\natexlab{}.
\newblock \showarticletitle{Corporate Social Responsibility via Multi-Armed Bandits}. In \bibinfo{booktitle}{\emph{Proceedings of the 2021 ACM Conference on Fairness, Accountability, and Transparency (FAccT '21)}} (Virtual Event, Canada). \bibinfo{publisher}{Association for Computing Machinery}, \bibinfo{pages}{26–40}.
\newblock
\showISBNx{9781450383097}


\bibitem[\protect\citeauthoryear{Sinha}{Sinha}{2024}]%
        {sinha2023textttbanditq}
\bibfield{author}{\bibinfo{person}{Abhishek Sinha}.} \bibinfo{year}{2024}\natexlab{}.
\newblock \showarticletitle{BanditQ: Fair Bandits with Guaranteed Rewards}. In \bibinfo{booktitle}{\emph{Proceedings of the Fortieth Conference on Uncertainty in Artificial Intelligence (UAI)}} \emph{(\bibinfo{series}{Proceedings of Machine Learning Research}, Vol.~\bibinfo{volume}{244})}. \bibinfo{publisher}{PMLR}, \bibinfo{pages}{3227--3244}.
\newblock


\bibitem[\protect\citeauthoryear{Slivkins}{Slivkins}{2019}]%
        {slivkins2024introductionmultiarmedbandits}
\bibfield{author}{\bibinfo{person}{Aleksandrs Slivkins}.} \bibinfo{year}{2019}\natexlab{}.
\newblock \showarticletitle{Introduction to Multi-Armed Bandits}.
\newblock \bibinfo{journal}{\emph{Foundations and Trends® in Machine Learning}} \bibinfo{volume}{12}, \bibinfo{number}{1-2} (\bibinfo{year}{2019}), \bibinfo{pages}{1--286}.
\newblock
\showISSN{1935-8237}


\bibitem[\protect\citeauthoryear{Wang, Bai, Sun, and Joachims}{Wang et~al\mbox{.}}{2021}]%
        {wang2021fairness}
\bibfield{author}{\bibinfo{person}{Lequn Wang}, \bibinfo{person}{Yiwei Bai}, \bibinfo{person}{Wen Sun}, {and} \bibinfo{person}{Thorsten Joachims}.} \bibinfo{year}{2021}\natexlab{}.
\newblock \showarticletitle{Fairness of exposure in stochastic bandits}. In \bibinfo{booktitle}{\emph{International Conference on Machine Learning (ICML)}}. \bibinfo{publisher}{PMLR}, \bibinfo{pages}{10686--10696}.
\newblock


\bibitem[\protect\citeauthoryear{Zhang, Vuong, and Luo}{Zhang et~al\mbox{.}}{2024}]%
        {Zhang2024NoRegretLF}
\bibfield{author}{\bibinfo{person}{Mengxiao Zhang}, \bibinfo{person}{Ramiro Deo-Campo Vuong}, {and} \bibinfo{person}{Haipeng Luo}.} \bibinfo{year}{2024}\natexlab{}.
\newblock \showarticletitle{No-Regret Learning for Fair Multi-Agent Social Welfare Optimization}. In \bibinfo{booktitle}{\emph{Neural Information Processing Systems (NeurIPS)}}, Vol.~\bibinfo{volume}{37}. \bibinfo{publisher}{Curran Associates, Inc.}
\newblock


\end{thebibliography}
